\documentclass{article}

\usepackage{microtype}
\usepackage{graphicx}
\usepackage{subcaption}
\usepackage{booktabs} 

\usepackage{hyperref}

\usepackage[dvipsnames]{xcolor}
\usepackage{xspace}
\usepackage{trimclip}

\definecolor{transientcolour}{HTML}{FDB196}
\definecolor{convergecolour}{HTML}{FBEAC8}
\definecolor{recurrentcolour}{HTML}{93C2C6}

\def\shrinkingintervals{{\color{OliveGreen}shrinking intervals}\xspace}

\DeclareTextFontCommand{\highlight}{\color{RedViolet}\em}
\newcommand{\mhighlight}[2][RedViolet!20]{\colorbox{#1}{$\displaystyle#2$}}

\usepackage[inline]{enumitem}

\usepackage{tikz,tikz-cd}
\usetikzlibrary{math,bayesnet,decorations.pathreplacing}

\usepackage{bm}
\usepackage{mathtools}
\usepackage{amsmath,amsthm,amsfonts,amssymb}

\theoremstyle{definition}
\newtheorem{definition}{Definition}
\newtheorem{assumption}{Assumption}

\theoremstyle{remark}
\newtheorem*{remark}{Remark}

\theoremstyle{remark}

\theoremstyle{plain}
\newtheorem{theorem}{Theorem}[section]
\newtheorem{proposition}{Proposition}[section]
\newtheorem{lemma}{Lemma}[section]
\newtheorem{corollary}{Corollary}[section]

\renewcommand{\emptyset}{\{\}}

\DeclareMathOperator*{\argmax}{arg\,max}
\DeclareMathOperator*{\argmin}{arg\,min}

\renewcommand{\Re}{\mathbb{R}}
\renewcommand{\SS}{\mathsf{X}}
\newcommand{\AS}{\mathsf{U}}

\newcommand{\SoR}[2][{}]{\mathsf{F}#1\!\left(#2\right)}

\newcommand{\PFsymbol}{\phi}
\newcommand{\PF}[2][{}]{\PFsymbol#1\!\left(#2\right)}

\newcommand{\DPFsymbol}{\psi}
\newcommand{\DPF}[3][{}]{\DPFsymbol#1\!\left(#2, #3\right)}

\newcommand{\subsof}[2][{}]{\mathfrak{B}#1\!\left(#2\right)}

\newcommand{\abs}[1]{\left\lvert #1 \right\rvert}
\newcommand{\card}[1]{\left\lvert #1 \right\rvert}
\newcommand{\supp}[1]{\textup{supp}\!\left(#1\right)}

\def\E(#1){\mathbb{E}\!\left[#1\right]}
\def\CE(#1|#2){\mathbb{E}\!\left[#1\,\middle\vert\,#2\right]}

\newcommand{\p}[2][{}]{\mathbb{P}#1\!\left(#2\right)}
\newcommand{\cp}[3][{}]{\mathbb{P}#1\!\left(#2\,\middle\vert\,#3\right)}

\makeatletter
\DeclareRobustCommand{\circbullet}{\mathbin{\vphantom{\circ}\text{\circbullet@}}}
\newcommand{\circbullet@}{%
    \check@mathfonts
    \m@th\ooalign{%
        \clipbox{0 0 0 {\dimexpr\height-\fontdimen22\textfont2}}{$\bullet$}\cr
        $\circ$\cr
    }%
}
\makeatother

\usepackage[accepted]{icml2022}

\icmltitlerunning{Reductive MDPs: A Perspective Beyond Temporal Horizons}

\begin{document}

\twocolumn[
\icmltitle{Reductive MDPs: A Perspective Beyond Temporal Horizons}

\icmlsetsymbol{equal}{*}
\icmlsetsymbol{past}{$\dagger$}

\begin{icmlauthorlist}
    \icmlauthor{Thomas Spooner}{jpm,past}
    \icmlauthor{Rui Silva}{jpm}
    \icmlauthor{Joshua Lockhart}{jpm}
    \icmlauthor{Jason Long}{jpm}
    \icmlauthor{Vacslav Glukhov}{jpm}
\end{icmlauthorlist}

\icmlaffiliation{jpm}{J.P.\ Morgan AI Research}
\icmlcorrespondingauthor{Thomas Spooner}{spooner10000@gmail.com}

\icmlkeywords{Markov Chains, Reinforcement Learning, Value Iteration, Drift Criteria}

\vskip 0.3in
]

\printAffiliationsAndNotice{} %

\begin{abstract}
    Solving general Markov decision processes (MDPs) is a computationally hard problem. Solving finite-horizon MDPs, on the other hand, is highly tractable with well known polynomial-time algorithms. What drives this extreme disparity, and do problems exist that lie between these diametrically opposed complexities? In this paper we identify and analyse a sub-class of stochastic shortest path problems (SSPs) for general state-action spaces whose dynamics satisfy a particular drift condition. This construction generalises the traditional, temporal notion of a horizon via decreasing reachability: a property called reductivity. It is shown that optimal policies can be recovered in polynomial-time for reductive SSPs---via an extension of backwards induction---with an efficient analogue in reductive MDPs. The practical considerations of the proposed approach are discussed, and numerical verification provided on a canonical optimal liquidation problem.
\end{abstract}

\section{Introduction}
The theory of Markov decision processes (MDPs) can broadly be divided into the study of finite- (FHMDP) and infinite-horizon problems. The latter have enjoyed a great deal of attention in the reinforcement learning (RL) literature as they capture a large class of problems~\cite{sutton:2018:reinforcement}, while assumptions such as ergodicity and unichain dynamics yield myriad tools of analysis from the study of Markov chains~\cite{arora:2012:deterministic}. The same is also true of the former, and indeed polynomial-time algorithms for finding an optimal policy have been known for polynomial horizons and finite state-spaces since the 1940s~\cite{arrow:1949:bayes,papadimitriou:1987:complexity}. While the story is more nuanced for binary-encoded horizons~\cite{fearnley:2015:complexity,balaji:2019:complexity}, one can efficiently solve many FHMDPs in practice using backwards induction (BI).

A natural question that arises is whether the performance benefits of BI are applicable to more general problems. To answer this, we introduce a property of absorbing Markov chains and MDPs that generalises the structural assumptions required by the backwards induction algorithm. Concretely, we identify a sufficient condition---reductivity---for (the transient part of) the transition matrix under a given policy to admit an upper-triangular form. This characterisation is based on the idea of uniformly decreasing reachability, is a sufficient condition for the existence of a Doeblin decomposition of the state-space~\cite{doeblin:1940:elements,tweedie:1993:doeblin}, and proffers a computationally tractable sub-class of stochastic shortest path (SSP) problems~\cite{bertsekas:1991:analysis,guillot:2020:stochastic}; see Figure~\ref{fig:classes}.

As a motivating example, consider the problem of driving a vehicle to a particular destination with limited fuel. There are no opportunities to refill the tank---you're in the countryside---and there is a terminal cost based on the remaining distance to the target location. This setting is naturally captured by the SSP framework and indeed similar problems are of active study in the literature; see e.g.~\cite{chen:2015:stochastic}. However, until recently, few results were known about the complexity of solving general indefinite-horizon SSPs, nor on the existence of low regret (online) learning algorithms~\cite{even:2005:experts,neu:2010:online,neu:2012:adversarial,rosenberg:2020:near,tarbouriech:2021:sample,chen:2021:finding}. In this paper, we show that SSPs with the aforementioned reductive property (as in the navigation example) admit a particularly simple polynomial-time algorithm based on a generalisation of backwards induction.

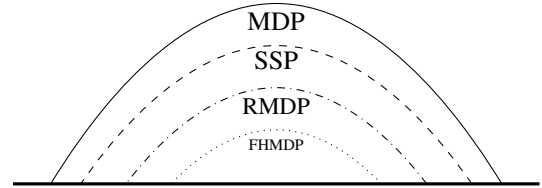
\begin{figure}[t]
    \centering
    \begin{tikzpicture}[yscale=0.8, xscale=1.0]
        \draw[very thick] (3.5,0) -- (-3.5,0);

        \draw (-3,0) parabola bend (0,3) (3,0);
        \node at (0,2.7) {MDP};

        \draw[dashed] (2.6,0) parabola bend (0,2.3) (-2.6,0);
        \node at (0,2.05) {SSP};

        \draw[dashdotted] (2,0) parabola bend (0,1.6) (-2,0);
        \node at (0,1.3) {\small RMDP};

        \draw[dotted] (1.4,0) parabola bend (0,0.9) (-1.4,0);
        \node at (0,0.65) {\tiny FHMDP};
    \end{tikzpicture}

    \caption{Class hierarchy of Markovian decision problems.}\label{fig:classes}
\end{figure}

\subsection*{Our Contributions}
\begin{enumerate}
    \item Propose a novel sub-class of SSPs, namely reductive MDPs (RMDPs), that characterise a large set of real-world problems; see Figure~\ref{fig:classes}.
    \item Analyse the complexity of said class and extend the backwards induction algorithm to applications therein.
    \item Perform numerical experiments on a canonical optimal liquidation problem, comparing performance to state-of-the-art value iteration algorithms.
\end{enumerate}

\subsection{Related Work}\label{sec:introduction:related_work}
\paragraph{Reachability}
Reachability is a diverse and richly studied topic in computer science that broadly pertains to the following question: ``given a system, can a particular state be reached for some initial condition?'' These kinds of queries feature in the context of deterministic control systems~\cite{sontag:2013:mathematical}, in Markov chains~\cite{akshay:2015:reachability}, MDPs~\cite{boutilier:1998:structured,haddad:2014:reachability,ashkenazi:2020:reachability}, and in graph theory~\cite{cormen:2009:introduction}. Reachability is also a prominent component in the analysis and verification of communication protocols~\cite{bochmann:1978:finite}. In this paper, we leverage the notion of reachability for a different purpose---namely, for the construction of a drift condition---but much insight on this setting can be drawn from existing work.

\paragraph{Value Iteration}
A great deal of research has been conducted on improving value iteration (VI) in finite-state MDPs, dating back to the earliest works on dynamic programming. In a seminal paper, \citet{bertsekas:1989:convergence} explored the convergence behaviour of asynchronous methods. Building on this, \citet{wingate:2004:p3vi} proposed a trio of enhancements including prioritisation, state partitioning~\cite{wingate:2003:efficient} and parallelisation; \citet{smith:2004:heuristic} simultaneously explored the use of heuristics. The themes covered in this work---such as the exploitation of structural properties of the MDP---have also featured in many influential papers~\cite{zang:2007:horizon,dai:2007:prioritizing,dai:2011:topological,grzes:2013:convergence}. We leverage many of these contributions, and propose an algorithm that complements the family of existing methods.

\section{Preliminaries}
We begin by introducing some standard notation and theory required to understand the content of this work: Markov chains (MCs), and Markov decision processes. Our nomenclature and notation broadly follow that of~\citet{meyn:2012:markov}. Vectors and matrices are written as $\bm{x}$ and $\bm{X}$, respectively; note the bold face. Sets, random variables and matrices are always in uppercase. Families of sets, such as (Borel) $\sigma$-algebras, will be denoted using variations on $\subsof{\cdot}$. Important quantities, such as the space of real values, will be expressed using blackboard font, or sans-serif.

\paragraph{Markov Chains}
An infinite-horizon, discrete-time MC is a sequence of random variables $\bm{X} \doteq \{X_n : n \in\mathbb{N}\}$ whose support occupy a state-space $\SS$. We assume time-homogeneity of $\bm{X}$, implying the existence of a transition semigroup $\{\p[_n]{x, B} : x\in\SS, B\in\subsof{\SS}\}$ whereby
\begin{align*}
    \p[_n]{x, B}
        &\doteq \cp{X_{m+n} \in B}{X_i, i < m, X_m = x}, \\
        &= \cp{X_{m+n} \in B}{X_m = x},
\end{align*}
denotes the (Markovian) probability of transitioning from the state $x$ to the set $B$ in $n$ steps from any time $m\geq 0$,\footnote{Note the lack of notational dependence on $m$. This is precisely the time-homogeneity property.} and the particular form of $\sigma$-algebra $\subsof{\SS}$ is dictated by the state-space~\citep[ch.~3]{meyn:2012:markov}. A state $x\in\SS$ is said to lead to a set $B\in\subsof{\SS}$ if there exists some $n > 0$ for which $\p[_n]{x, B} > 0$. The time of first return to a set $B$ is then defined as $\tau_B \doteq \min\!\left\{ n \geq 1 : X_n \in B \right\}$ such that the probability of a state $x$ ever leading to the set $B$ is
\begin{equation}
    \p[_\infty]{x, B} \doteq \cp{\tau_B < \infty}{X_0 = x}.
\end{equation}
We also define the collection of \highlight{absorbing and indecomposable (a.i.)} subsets of a family $\subsof{A}$ as
\begin{equation}\label{eq:borel:bullet}
    \boxed{\subsof[_\bullet]{A} \doteq \left\{B \in \subsof{A} : B \textrm{ a.i.} \right\}}
\end{equation}
for $A \subseteq \SS$~\cite{tweedie:1993:doeblin}. The absorbing property ensures that, for any such subset $B \in \subsof[_\bullet]{\SS}$, the chain remains almost surely (a.s.) upon entry; that is, for $C\subset \SS$ with $C \cap B = \emptyset$, we may assert that $\p[_\infty]{x, B} = 1$ and $\p[_\infty]{x, C} = 0$ for all $x \in B$. The indecomposability property then ensures that each $B$ contains no disjoint pair of absorbing subsets. This allows us to perform meaningful decompositions of $\SS$ via complete covers. For brevity, let
\begin{equation}
    \boxed{\SS_\bullet \doteq \bigcup \subsof[_\bullet]{\SS} \quad\text{and}\quad \SS_\circ \doteq \SS \setminus \SS_\bullet.}
\end{equation}
Finally, we recall that an \highlight{inessential} set $A\in\subsof{\SS}$ is such that, from any state $x\in\SS$, the probability of visiting $A$ infinitely often is zero~\cite{doeblin:1940:elements}.

\paragraph{Markov Decision Processes}
MDPs extend MCs by the inclusion of an action space $\AS$ and a reward function $r : \SS\times\AS\times\SS \to \Re$. The former gives rise to a joint state-action process $(\bm{X}, \bm{U})$ wherein the state-transition kernel (i.e.\ probability of transitioning from one state to a successor set) conditions on the state and action: $\p[_1]{x, B} = \int_{\AS(x)}\cp[_1]{x, B}{u} \, \textup{d}\cp{u}{x}$ for $x\in\SS$ and $B\in\subsof{\SS}$; the term $\AS\!\left(x\right)$ allows for action masking. At each time $t \geq 0$, the agent selects an action, denoted $U_t$, by sampling its policy $\pi \in \Pi \doteq \left\{ \SS \to \mathcal{P}\!\left(\AS\right) \right\}$ in order to achieve a pre-specified goal. Typically, this is to maximise the (discounted) sum of future rewards as quantified by the action-value function
\begin{displaymath}
    q_\pi\!\left(x, u\right) \doteq \mathbb{E}_\pi\!\left[ \sum_{\tau = t}^\infty \gamma^{\tau-t} r\!\left(X_\tau, U_\tau, X_{\tau+1}\right) \; \middle\vert \; \begin{aligned}X_t &= x \\ U_t &= u\end{aligned} \right]
\end{displaymath}
and value function $v_\pi\!\left(x\right) \doteq \mathbb{E}\left[q_\pi\!\left(x, u\right) \;\middle\vert\; u\sim\pi\!\left(x\right)\right]$, where $\gamma \in [0, 1]$. Maximising these quantities with respect to $\pi$ is well-known to yield an optimal policy $\pi^\star$, and is indeed at the foundation of value iteration methods~\cite{puterman:2014:markov}.

\section{Forward Reachability}
For any Markov chain $\bm{X}$, we can characterise the \highlight{$n$-step reachable state set} from a point $x\in\SS$ as the collection
\begin{equation}\label{eq:reachable_states:n}
    \boxed{\SoR[_n]{x} \doteq \bigcap \left\{B \in \subsof{\SS} : \p[_n]{x, B} = 1 \right\},}
\end{equation}
for $n > 0$ and with $\SoR[_0]{x} \doteq \left\{x\right\}$. This describes the minimal subset of states that are accessible from $x$.\footnote{An intersection of almost sure sets, as opposed to a union over non-null sets, avoids issues of spurious inclusion.} The \highlight{reachable state set} is then defined as the union thereof,
\begin{equation}\label{eq:reachable_states}
    \boxed{\SoR{x} \doteq \bigcup\nolimits_{n = 0}^\infty \SoR[_n]{x},}
\end{equation}
where, for any given state, it follows that $\SoR[_n]{x} \subseteq \SoR{x} \subseteq \SS$ for all $n \geq 0$. Note that \eqref{eq:reachable_states} does not necessarily define a closed set of states, nor a communicating class, since forward reachability is uni-directional and $\SoR{x}$ includes $\SoR[_0]{x}$ even if $\p[_\infty]{x, \{x\}} = 0$. It is, however, true that $\p[_\infty]{x, \SoR{x}\setminus\{x\}} = 1$ for all $x\in\SS$, relating \eqref{eq:reachable_states:n} and \eqref{eq:reachable_states} to return times~\cite{meyn:2012:markov}. Finally, we define
\begin{equation}
    \SoR{A} \doteq \bigcup\nolimits_{x\in A} \SoR{x}
\end{equation}
as the collection of all reachable states from a set $A\subseteq\SS$.

\subsection{Reachability Potentials}
A central theme of this work rests in the quantification of the number of states that are reachable from different points in $\SS$. Doing so for a wide class of processes thus motivates a measure-theoretic treatment. Consider a chain diffusing on the unit interval, $\SS \doteq [0, 1]$, where $\SoR[_1]{x} \subseteq [0, x)$ for all $x\in\SS$ and $\SoR{x} \setminus \SoR[_0]{x} = \SoR[_1]{x}$; we refer to this throughout as the \shrinkingintervals{} problem (see Section~\ref{sec:shrinking_intervals} in the appendix for a visualisation). The natural choice of measure in this case is the Lebesgue measure, $\ell$, yielding $\ell\circ\SoR{x} = x$. In the following, we formalise this concept as a natural class of potential functions that are defined on $\SS$.

For a measure $\mu : \subsof{\SS} \to [0, \infty)$, we define a (reachability) \highlight{potential function $\PFsymbol_\mu : \SS \to [0, \infty)$} as the composition
\begin{equation}\label{eq:potential}
    \boxed{\PF[_\mu]{x} \doteq \mu\circ\SoR{x}.}
\end{equation}
The output of $\PFsymbol_\mu$ characterises the number of reachable states from a point $x\in\SS$; though the meaning varies based on $\mu$. From this, one can define a function over transitions,
\begin{equation}\label{eq:dpf}
    \boxed{\DPF[_\mu]{x}{x'} \doteq \PF[_\mu]{x'} - \PF[_\mu]{x},}
\end{equation}
which, loosely speaking, is the difference in the number of reachable states between $x$ and $x'$ with respect to the chosen measure.
We note that $\DPFsymbol_\mu$ can always be decomposed by $\sigma$-additivity to give $\DPF[_\mu]{x}{x'} = \mu\!\left(\SoR{x'} \setminus \SoR{x}\right) - \mu\!\left(\SoR{x} \setminus \SoR{x'}\right)$, which is often easier to compute than \eqref{eq:dpf} directly. Indeed, since $x'\in\SoR[_1]{x}$ a.s., it follows that
\begin{displaymath}
    \DPF[_\mu]{x}{x'} = -\mu\!\left(\SoR{x} \setminus \SoR{x'}\right) \leq 0.
\end{displaymath}


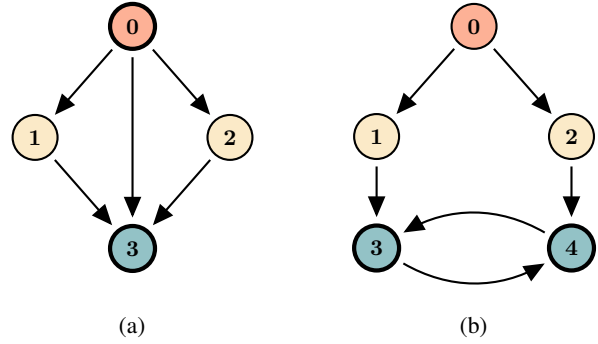
\begin{figure}[t]
    \centering
    \begin{subfigure}[t]{0.45\linewidth}
        \centering
        \begin{tikzpicture}[scale=0.85, transform shape]
            \draw[white] (-1,-4.25) rectangle (1,0);
            \node[latent, ultra thick, fill=transientcolour] (0) at (0,0) {$\bm{0}$};

            \node[latent, draw=white, below=of 0] (ghost) {};
            \node[latent, thick, fill=convergecolour, left=.8 of ghost] (1) {$\bm{1}$};
            \node[latent, thick, fill=convergecolour, right=.8 of ghost] (2) {$\bm{2}$};

            \node[latent, ultra thick, fill=recurrentcolour, below=of ghost] (3) {$\bm{3}$};

            \edge[thick, shorten >= 2.5pt, shorten <= 2.5pt] {0} {1, 2, 3};
            \edge[thick, shorten >= 2.5pt, shorten <= 2.5pt] {1, 2} {3};
        \end{tikzpicture}
        \caption{}\label{fig:rmcs:left}
    \end{subfigure}
    \hfill
    \begin{subfigure}[t]{0.45\linewidth}
        \centering
        \begin{tikzpicture}[scale=0.85, transform shape]
            \draw[white] (-1,-4.25) rectangle (1,0);
            \node[latent, thick, fill=transientcolour] (0) at (0,0) {$\bm{0}$};

            \node[latent, draw=white, below=of 0] (ghost) {};
            \node[latent, thick, fill=convergecolour, left=.8 of ghost] (1) {$\bm{1}$};
            \node[latent, thick, fill=convergecolour, right=.8 of ghost] (2) {$\bm{2}$};

            \node[latent, ultra thick, fill=recurrentcolour, below=of 1] (3) {$\bm{3}$};
            \node[latent, ultra thick, fill=recurrentcolour, below=of 2] (4) {$\bm{4}$};

            \edge[thick, shorten >= 2.5pt, shorten <= 2.5pt] {0} {1, 2};
            \edge[thick, shorten >= 2.5pt, shorten <= 2.5pt] {1} {3};
            \edge[thick, shorten >= 2.5pt, shorten <= 2.5pt] {2} {4};

            \path[->, thick, shorten >= 2.5pt, shorten <= 2.5pt] (3) edge [bend right] (4);
            \path[->, thick, shorten >= 2.5pt, shorten <= 2.5pt] (4) edge [bend right] (3);
        \end{tikzpicture}
        \caption{}\label{fig:rmcs:right}
    \end{subfigure}

    \caption{Finite MCs with the reductive property under the counting measure. {\color{transientcolour}Red} and {\color{convergecolour}yellow} nodes are transient, {\color{recurrentcolour}blue} nodes are recurrent, and thick borders indicate the existence of a self-loop.}
    \label{fig:rmcs}
\end{figure}

\section{Reductive Processes}\label{sec:rmcs}
A reductive process is, in essence, any process where the set of reachable states shrinks over time; see Figure~\ref{fig:rmcs}. A finite-horizon chain, for example, is an important case where this property holds: time inexorably moves forward. It also encompasses the loop-free SSP setting studied by, e.g., \citet{neu:2010:online} in which $\SS$ is composed of ``layers.'' Yet there are many other problems that have equivalent behaviour, such as finite resource domains and those for which there is a strong sense of irreversibility; e.g.\ navigation with limited fuel, or optimal liquidation. To formalise this, we first identify those \highlight{non-absorbing states that exhibit self-loops},
\begin{equation}\label{eq:loops}
    \mathsf{L}\!\left(A\right) \doteq \left\{ x \in A : 0 < \p{x, \{x\}} < 1 \right\}
\end{equation}
with $A\subseteq\SS$, and then provide the following definitions of reductive Markov chains and decision processes.

\begin{definition}[RMC]\label{def:rmc}
    Let $\mu : \subsof{\SS} \to [0, \infty)$ be a measure. Then, an MC is said to be $\mu$-reductive if $\SS_\bullet \ne \emptyset$ and
    \begin{equation}\label{eq:rmcs}
        \begin{cases}
            \DPF[_\mu]{x}{x'} = 0 &\quad \textrm{for } x' \in \SS_\bullet \cup \mathsf{L}\!\left(\{x\}\right), \\
            \DPF[_\mu]{x}{x'} < 0 &\quad \textrm{otherwise},
        \end{cases}
    \end{equation}
    holds for all states $x\in\SS$ and their successors $x'\in\SoR[_1]{x}$.
\end{definition}

\begin{remark}
    There exist a large class of vacuously reductive MCs where $\SS_\bullet = \SS$; e.g.\ the random walk on $\mathbb{Z}$. These are not our interest, but precluding such cases from Definition~\ref{def:rmc} would be contrived. In the remainder of the paper we will focus on ``non-trivial'' instances of reductivity.
\end{remark}

\begin{definition}[RMDP]\label{def:rmdp}
    An MDP is said to be $\mu$-reductive if every policy induces a $\mu$-RMC.
\end{definition}

These definitions deserve some attention. First, observe that the \highlight{equality condition} in \eqref{eq:rmcs} ensures that $\SS_\bullet$ corresponds to a fixed point of $\phi_\mu$ under the forward difference operator (see Section~\ref{sec:rmcs:convergence}) while permitting self-loops in the transient set $\SS_\circ$, provided they have probability less than one; thus $\mathsf{L}\!\left(\SS\right)$ need not be empty. Note that any state for which $\p{x, \{x\}} = 1$ is necessarily a.i.\ and would therefore occupy $\SS_\bullet$, hence the distinction from $\mathsf{L}\!\left(\SS\right)$. Second, the \highlight{inequality condition} guarantees that any other transition that can occur must strictly reduce the potential of the system. These two conditions are always well defined given the constraint that $\mu\!\left(\SS\right)$ be finite and form a \highlight{sufficient condition for a Doeblin decomposition on $\SS$}; see Theorem~\ref{thm:decomp} below.\footnote{This result can be seen as a special case of the famous Th\'{e}or\`{e}me I de \citet{doeblin:1940:elements}.}

\begin{theorem}[Decomposition]\label{thm:decomp}
    For any $\mu$-RMC, the set $\SS_\circ$ is inessential and $\SS_\bullet \cup \SS_\circ$ forms a Doeblin decomposition with possibly uncountable set $\subsof[_\bullet]{\SS}$.
\end{theorem}
\begin{proof}[Proof]
    Let $\bm{X}$ denote a $\mu$-RMC and $\bm{x} \in \supp{\bm{X}}$ a realisation of the process. The induced sequence of potentials $\bm{\phi}_\mu \doteq \left\{ \PF[_\mu]{x_n} \right\}_{n\geq 0}$ must obey \eqref{eq:rmcs}, by virtue of reductivity, which implies that $\bm{\phi}_\mu$ is monotone decreasing and bounded from below as $\mu$ is finite and unsigned by construction. Invoking the monotone convergence theorem, we see that $s \doteq \lim_{n\to\infty} \PF[_\mu]{x_n} = \inf_n \PF[_\mu]{x_n}$ and thus $\lim_{n\to\infty} \DPF[_\mu]{x_n}{x_{n+1}} = s - s = 0$ by the difference law of limits. Since non-absorbing self-loops a.s.\ occur a finite number of times by \eqref{eq:loops}, it must also be that the state process reaches an absorbing and indecomposable a.s.. Finally, by observing that the arguments presented above apply to all trajectories in $\supp{\bm{X}}$, we may conclude that the original claim holds.
\end{proof}

Interestingly, the choice to take $\SS_\bullet$ as the set of successor states with stationary potential is unique in the sense that no other subset of $\SS$ could satisfy \eqref{eq:rmcs}. This completeness property of the RMC definition itself, which is stated in Proposition~\ref{prop:completeness} below, is not entirely surprising---it is an artefact of the formative results first presented by \citet{doeblin:1940:elements} some 80 years ago. Nevertheless, this observation helps justify our construction and suggests that reductivity is a restrictive, but fundamental property of a chain.

\begin{proposition}[Completeness]\label{prop:completeness}
    Suppose we replaced $\SS_\bullet$ in Definition~\ref{def:rmc} with any other subset $B \subset \SS$. Then no MC exists that can satisfy the candidate drift criterion.
\end{proposition}

\subsection{Convergence, Stability and Level-Sets}\label{sec:rmcs:convergence}
The first question we might ask about the behaviour of reductive processes is whether they are stable or converge in some sense. It is well understood that absorbing MCs---of which RMCs are a subclass---on finite state-spaces eventually reach $\SS_\bullet$~\cite{kemeny:1976:finite}. It is also known that the state process of any chain admitting a Doeblin decomposition will eventually reach one of the a.i.\ subsets~\cite{doeblin:1940:elements}.\footnote{A sufficient condition for this decomposition is that there exists a finite measure $\nu$ attributing positive mass to each absorbing subset of $\SS$~\cite{tweedie:1993:doeblin}.} As a third angle of attack, we remark that Eq. \eqref{eq:rmcs} is nothing but a negative drift condition with Lyapunov function $\PFsymbol_\mu$~\citep[ch.~11,19]{meyn:2012:markov}. In Corollary~\ref{corr:convergence} below we consolidate these observations, showing that the support of \highlight{an RMC $\bm{X}$ eventually reaches the a.i.\ subspace $\SS_\bullet$ regardless of $\mu$}.

\begin{corollary}[Convergence]\label{corr:convergence}
    For any $\mu$-RMC, the support process $\bm{M} \doteq \{M_n \doteq \supp{X_n}\}_{n \geq 0}$ converges almost surely to an a.i.\ subset: $\p{\lim_{n\to\infty} M_n \in \subsof[_\bullet]{\SS}} = 1$.
\end{corollary}
\begin{proof}
    The claim follows from Theorem~\ref{thm:decomp} and the fact that $\SS_\bullet$ is formed by a cover of the a.i.\ sets $\subsof[_\bullet]{\SS}$.
\end{proof}

\paragraph{Example.}
In the \shrinkingintervals{} problem, any initial condition $x\in\SS$ yields a sequence $X_0 \doteq x, X_1, X_2, \dots$ that converges monotonically and almost surely to the infimum value $x_\infty \in [0, 1)$. For instance, if we let $\bm{Y}$ denote an i.i.d.\ disturbance sequence of uniformly distributed random variables on the half-interval $[0, 1)$, then the chain defined by $X_{n+1} \doteq Y_{n+1} \cdot X_n$ with $X_0 \doteq 1$ converges a.s.\ to $x_\infty = 0$; see Lemma~\ref{lem:shrinking} in the appendix. Yet, while the finite-measure coverage condition first proposed by \citet{doeblin:1940:elements} does not hold for $\ell$ in this case, we can assert that: (a) $\SS_\bullet = \left\{x_\infty\right\}$ is a.i.; (b) $\SS_\circ = \left(x_\infty, 1\right]$ is inessential; and (c) that $\SS_\bullet$ and $\SS_\circ$ are disjoint. These together ensure that Corollary~\ref{corr:convergence} holds as anticipated with $\SS = [x_\infty, 1]$.

\paragraph{Potential Level-Sets.}
For further intuition, let us ponder the structure of $\SS$ imposed by the potential. The subspace $\SS_\bullet$ is constructed as a union over a.i.\ subsets of $\SS$, that is $B \subseteq \SS_\bullet$ for all $B \in \subsof[_\bullet]{\SS}$. There must therefore exist constants $\left\{b_B\right\}_{B\in\subsof[_\bullet]{\SS}}$ such that $\PF[_\mu]{B} = b_B = \PF[_\mu]{x}$ for all states $x \in B$. This property extends as a lower bound to the set of predecessor states. Concretely, if we define
\begin{equation}\label{eq:g_map}
    \mathsf{G}_n\!\left(A\right) \doteq \left\{x\in\SS\setminus A : \sum\nolimits_{k=1}^n \p[_k]{x, A} > 0\right\}
\end{equation}
as the states that lead to $A \in \subsof{\SS}$ (excluding self-loops) in up to $n > 0$ steps, then $b_B \leq \PFsymbol_\mu \circ \mathsf{G}_n\!\left(B\right)$. The bound is recursive in the sense that, for a given $B\in\subsof{\SS}$,
\begin{equation}\label{eq:inf_sup}
    \inf_{x' \in \mathsf{G}_m\!\left(B\right)} \PF[_\mu]{x'} \geq \sup_{x \in \mathsf{G}_n\!\left(B\right)} \PF[_\mu]{x}
\end{equation}
for all $m > n > 0$. This property of all Markov chains is a consequence of the fact that $\SoR{X_0} \supseteq \SoR{X_1}  \supseteq \SoR{X_2} \cdots$ is almost surely a monotone sequence. The key distinction between MCs and RMCs lies in the conditions for which the inequality in \eqref{eq:inf_sup} is strict.

\begin{figure}[t]
    \centering
    \begin{tikzpicture}
        \tikzmath{\msize = 4;}

        \draw[thick] (0,0) rectangle (\msize, \msize);

        \draw[thick, dashed] (0, 0.3*\msize) -- (0.7*\msize, 0.3*\msize);
        \draw (1.1*\msize/3, 0.15*\msize) node {$\bm{0}$};
        \draw (0.7*\msize/3, 0.5*\msize) node {$\bm{0}$};

        \draw[thick, fill=transientcolour, miter limit=1] (0, \msize) -- (0.7*\msize, 0.3*\msize) -- (0.7*\msize, \msize) -- (0, \msize) -- cycle;
        \draw (1.5*\msize/3, 2.3*\msize/3) node {$\bm{P}_\circ$};

        \draw[thick, fill=convergecolour] (0.7*\msize, \msize) -- (0.7*\msize, 0.3*\msize) -- (\msize, 0.3*\msize) -- (\msize, \msize) -- (0, \msize) -- cycle;
        \draw (2.55*\msize/3, 2*\msize/3) node {$\bm{P}_{\circbullet}$};

        \draw[thick, fill=recurrentcolour] (0.7*\msize, 0.3*\msize) -- (\msize, 0.3*\msize) -- (\msize, 0) -- (0.7*\msize, 0) -- (0.7*\msize, 0.3*\msize) -- cycle;
        \draw (0.85*\msize, 0.15*\msize) node {$\bm{P}_\bullet$};
    \end{tikzpicture}

    \caption{Canonical decomposition of the transition dynamics matrix $\bm{P}$ into {\color{transientcolour}trans}{\color{convergecolour}ient} and {\color{recurrentcolour}recurrent} sets for reductive Markov chains.}
    \label{fig:block_decomp}
\end{figure}
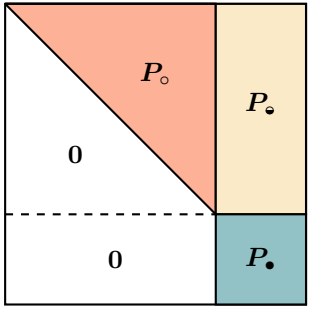

\subsection{Monotone Transition Dynamics}\label{sec:rmcs:monotone}
Take a $\mu$-RMC and impose a \highlight{binary relation $\preccurlyeq_\mu$} over the state-space where, for all $x,x'\in\SS$,
\begin{equation}\label{eq:preorder}
    \boxed{x \preccurlyeq_\mu x' \iff \phi_\mu\!\left(x\right) \leq \phi_\mu\!\left(x'\right).}
\end{equation}
This defines a total preorder over the state-space\footnote{This relation does not satisfy antisymmetry and therefore cannot qualify as a partial order.} that aligns directly with the ``volumes'' of the reachable sets. When combined with the definition of an RMC, \eqref{eq:preorder} is a sufficient condition for the existence of a refined canonical form for the dynamics matrix, $\bm{P}$, in finite-state Markov chains~\cite{kemeny:1976:finite}. As illustrated in Figure~\ref{fig:block_decomp}, the transient subspace \highlight{$\SS_\circ$ exhibits an upper-triangular structure} which can be leveraged during value iteration much as in traditional finite-horizon settings, as will be shown in Section~\ref{sec:algorithms}. A formal statement of this result is given below.

\begin{definition}[Finite-State Reductivity]\label{def:frmc}
    A finite-state reductive MC (FRMC) is a Markov chain on $\SS$ with $\card{\SS} = M < \infty$ that satisfies \eqref{eq:rmcs} under the counting measure. A finite-state reductive MDP (FRMDP) is defined analogously.
\end{definition}

\begin{lemma}\label{lem:dag_with_loops}
    The directed graph associated with $\SS_\circ$ of an RMC has self-loops on the states $\mathsf{L}\!\left(\SS\right)$ and no other cycles.
\end{lemma}

\begin{theorem}[Canonical Form]\label{thm:canonical}
    For any FRMC, there exists a permutation of $\SS$ such that $\bm{P} = \begin{bsmallmatrix} \bm{P}_\circ & \bm{P}_{\circbullet} \\ \bm{0} & \bm{P}_\bullet \end{bsmallmatrix}$, where $\bm{P}_\circ$ is an $\card{\SS_\circ}$-by-$\card{\SS_\circ}$ upper-triangular matrix, $\bm{P}_\bullet$ is an $\card{\SS_\bullet}$-by-$\card{\SS_\bullet}$ block-diagonal matrix, and $\bm{P}_{\circbullet}$ is $\card{\SS_\circ}$-by-$\card{\SS_\bullet}$.
\end{theorem}
\begin{proof}[Proof]
    Theorem~\ref{thm:decomp} allows us to decompose $\SS$ into $\SS_\circ \cup \SS_\bullet$. By Lemma~\ref{lem:dag_with_loops} we can also assert that $\SS_\circ$ corresponds to a directed graph with self-loops: $\mathcal{G} \doteq \left(\SS_\circ, \mathcal{E}\right)$, where $\mathcal{E} \doteq \left\{ (x, x') : x \in \SS_\circ, x'\in\SoR[_1]{x}\setminus\SS_\bullet \right\}$. Since it is well known that all directed acyclic graphs are isomorphic to graphs with strictly upper-triangular adjacency matrices, it follows that the same is true of the graph $\mathcal{G}' \doteq \left(\SS_\circ, \mathcal{E}'\right)$ with $\mathcal{E}' \doteq \mathcal{E} \setminus \left\{ (x, x) : x \in \mathsf{L}\!\left(\SS\right) \right\}$. The removed edges, being only self-loops, then only contribute diagonal entries which yields a non-strict upper-triangular matrix. As no special criteria were claimed about $\SS_\bullet$, we can simply compose the adjacency matrices together to yield the desired result, thus concluding the proof.
\end{proof}

This result suggests that we need only prove that a finite-state chain is reductive in order to assert that the canonical decomposition is satisfied. Indeed the reverse also holds.

\begin{corollary}\label{corr:canon_to_reduct}
    Any finite-state absorbing MC that can be expressed in canonical form is necessarily an FRMC.
\end{corollary}

Theorem~\ref{thm:canonical} and Corollary~\ref{corr:canon_to_reduct} provide two angles to attack the same problem: prove the dynamics matrix has canonical form; or that the chain is reductive under the counting measure. As a concrete example, consider the FRMC in Figure~\ref{fig:rmcs:right}. We can verify that this chain is reductive simply from inspection and, indeed, examination of the dynamics matrix only serves to confirm that the canonical form is upheld with $\SS_\circ = \left\{0, 1, 2\right\}$ and $\SS_\bullet = \left\{3, 4\right\}$. For large problems, however, inspecting $\bm{P}$ directly may not be feasible, and finding such a permutation of the induced graph would require a breadth-first search; i.e.\ an operation with polynomial complexity in $\card{\SS}$~\cite{cormen:2009:introduction}. While we leave these specific questions to future work, we remark that proving an MC satisfies \eqref{eq:rmcs} when domain knowledge is available---i.e.\ the settings we address---does not typically require explicit enumeration of a graph; see Section~\ref{sec:optimal_liq}.

\paragraph{General State Spaces}
The analysis thus far has focused on FRMCs as they permit convenient analysis in terms of transition matrices. In the general case, however, we no longer have such a structure and must instead reason about the functional properties of a transition kernel~\cite{klenke:2013:probability}. While it is beyond the scope of this paper to do a full exposition of this, we note that in many cases it is just as easy to verify \eqref{eq:rmcs} in continuous domains as it is for FRMCs. In the \shrinkingintervals problem, for example, we constrain the set of viable kernels such that the support remains bounded as $\supp{X_{i+1}} = [0, X_i)$ for all $i \geq 0$.

\section{Reductive Value Iteration}\label{sec:algorithms}
Many algorithms have been developed over the last 20 years to improve upon the performance of VI in (general) finite-state MDPs; see Section~\ref{sec:introduction:related_work}. However, each of these methods---such as prioritised VI (PVI), topological VI (TVI) and backward VI (BVI)---comes with its own set of limitations~\citep[ch.~3]{kolobov:2012:planning}, with worst-case run-time that is exponential in $\card{\SS}$. In this paper we take an approach that is inspired by the aforementioned algorithms, but leverages the structural properties of RMDPs for improved performance in the presence of non-trivial reductivity.

The proposed variant of VI---named \highlight{reductive VI (RVI)}---proceeds similarly to the traditional implementation for finite-horizon problems:
\begin{enumerate*}[label=(\arabic*)]
    \item solve the a.i.\ subspace $\SS_\bullet$ using your favourite VI algorithm; and
    \item recurse through the potential level-sets, solving each group of transient states through one-step lookaheads.
\end{enumerate*}
This can be seen as a strict generalisation of the reverse VI algorithm of \citet{zang:2007:horizon}, and a special case of backward VI~\cite{dai:2007:prioritizing}; the full process is stated in Algorithm~\ref{alg:rvi}. Indeed, in the case where $r\!\left(x\right) = 0$ for all $x\in\SS_\bullet$ the first part is trivial as $v\!\left(x\right) = 0$ holds uniformly across $\SS_\bullet$. We can also show that the ``self-loop'' recursion associated with forward lookahead used to evaluate $q(x, u)$ in step 9 can be avoided.

Recall that the action-value function $q(x, u)$ is given by an expectation over successor states~\cite{puterman:2014:markov}: $\int_{\SoR[_1]{x}} \left[r(x, u, x') + \gamma v\!\left(x'\right)\right] \, \textup{d}\cp{x'}{x, u}$. By the definition of RMCs, $\SoR[_1]{x}$ must also comprise a set $A$ excluding $x$ or $A \cup \left\{x\right\}$. We may thus \highlight{decompose $q(x, u)$} by separating terms according to the transiency of their implied trajectories. As shown in Lemma~\ref{lem:q_decomp} in the appendix,
\begin{equation}\label{eq:q_decomp}
    \boxed{\begin{aligned}
        q(x, u) = r(x, u)
            &+ \gamma \cdot \beta(x, u) \cdot q_\circ'(x, u) \\
            &+ \gamma \cdot \alpha(x, u) \cdot q_\bullet'(x, u),
    \end{aligned}}
\end{equation}
where, for the event map $E(x, u) \doteq \{X_t = x, U_t = u\}$,
\begin{align*}
    \alpha(x) &\doteq \cp[_\pi]{X_{t+1} = x}{X_t = x} \doteq 1 - \beta(x), \\
    \alpha(x, u) &\doteq \cp[_\pi]{X_{t+1} = x}{E(x, u)} \doteq 1 - \beta(x, u), \\
    v_\circ'(x) &\doteq \mathbb{E}_\pi\!\left[v(X_{t+1}) ~\middle\vert~ X_t = x, X_{t+1} \ne x\right], \\
    q_\circ'(x, u) &\doteq \mathbb{E}_\pi\!\left[v(X_{t+1}) ~\middle\vert~ E(x, u), X_{t+1} \ne x\right], \\
    q_\bullet'(x, u) &\doteq \frac{r(x) + \gamma \cdot \beta(x) \cdot v_\circ'(x)}{1 - \gamma \cdot \alpha(x)},
\end{align*}
and $r(x) \doteq \mathbb{E}_\pi\!\left[R_{t+1} ~\middle\vert~ X_t = x\right]$. While there are many moving parts to \eqref{eq:q_decomp}, the intuition is clear and it can be evaluated efficiently. For example, $q_\bullet'(x, u)$ and $q_\circ'(x, u)$ correspond to the expected future values given that we either repeat or leave $x$, respectively The advantage of the latter expression is that it depends only on known quantities: it is defined explicitly without recursion; hence the quotient form. Finally, we note that when $\alpha(x) = 0$ or $\beta(x) = 0$, you recover the classic expected forms: $r(x, u) + \gamma\cdot q_\circ'(x, u)$, and $r(x, u) + \frac{\gamma}{1 - \gamma} \cdot r(x)$, respectively.

\begin{algorithm}[t]
    \caption{Reductive Value Iteration (RVI)}\label{alg:rvi}

    \begin{algorithmic}[1]
        \STATE \textbf{Input}: Level-set oracle $\mathcal{M}_\phi$ and the absorbing set $\SS_\bullet$.
        \vspace{0.35em}
        \STATE Init $\pi\!\left(x\right)$, $v\!\left(x\right)$, $q\!\left(x, u\right)$ arbitrarily for $(x, u) \in \SS\times\AS$.
        \STATE Solve $\SS_\bullet$ using VI and assign $\SS_\star \gets \SS_\bullet$.
        \vspace{0.35em}
        \WHILE{$\SS_\star \ne \SS$}
            \STATE $A \gets \mathcal{M}_\phi\!\left(\SS\setminus\SS_\star\right)$.
            \vspace{0.35em}
            \FOR{$x \in A$}
                \STATE Update $q\!\left(x, u\right)$ for all $u\in\AS(x)$ using \eqref{eq:q_decomp}.
                \vspace{0.35em}
                \STATE $\pi\!\left(x\right) \gets \argmax_{u\in\AS(x)} q\!\left(x, u\right)$.
                \STATE $v\!\left(x\right) \gets q\!\left(x, \pi\!\left(x\right)\right)$.
            \ENDFOR
            \vspace{0.35em}
            \STATE $\SS_\star \gets \SS_\star \cup A$.
        \ENDWHILE
        \vspace{0.35em}
        \STATE \textbf{Result}: Optimal value functions $(v^\star, q^\star)$ and policy $\pi^\star$.
    \end{algorithmic}
\end{algorithm}

\subsection{Complexity}
There are two key properties of FRMDPs under the constraint of deterministic policies that reduce the complexity of RVI. First, we can replace the integrals in \eqref{eq:q_decomp} with summations. Second, we can remove all expectations w.r.t.\ the policy distribution. This means that $v(x)$ reduces to $q(x, \pi(x))$ where $\pi(x)$ denotes, with some abuse of notation, the only action with positive probability mass. This substantially reduces the complexity of the algorithm, especially in cases where $\AS$ is of large cardinality.

\highlight{Reductive VI is also a single-pass algorithm} in the sense that we only perform a single iteration of the ``inner loop;'' there is no dependence on some convergence criterion for termination, such as an $\varepsilon$ residual bound. Indeed, RVI is very closely related to the classical backward induction algorithm used for FHMDPs, relying crucially on the order in which one performs the one-step lookahead updates; see steps 5 and 11 of Algorithm~\ref{alg:rvi}. To summarise, we show below that under mild computational assumptions on $\PFsymbol$, the procedure underling RVI converges to the optimal value function in polynomial-time.

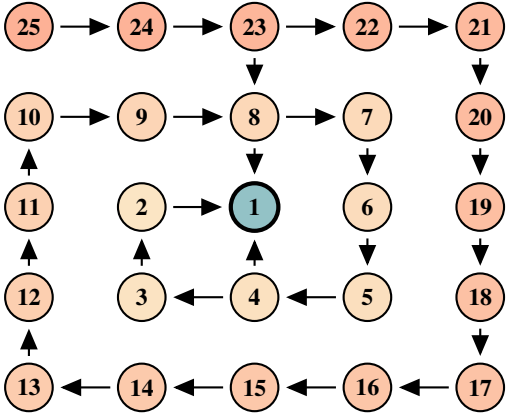
\begin{figure}
    \centering
    \begin{tikzpicture}[thick, yscale=-1]  
        \def\Ex{1.5};
        \def\Ey{1.2};
        \def\CircleSize{18};
        
        \def\Edges{0/0/1/0,
                   1/0/2/0,
                   2/0/3/0, 2/0/2/1,
                   3/0/4/0,
                   4/0/4/1,
                   0/1/1/1,
                   1/1/2/1,
                   2/1/3/1,
                   2/1/2/2,
                   3/1/3/2,
                   4/1/4/2,
                   0/2/0/1,
                   1/2/2/2,
                   3/2/3/3,
                   4/2/4/3,
                   0/3/0/2,
                   1/3/1/2,
                   2/3/1/3, 2/3/2/2,
                   3/3/2/3,
                   4/3/4/4,
                   0/4/0/3,
                   1/4/0/4,
                   2/4/1/4,
                   3/4/2/4,
                   4/4/3/4};
        \def\Potential{0/0/25,
                       1/0/24,
                       2/0/23,
                       3/0/22,
                       4/0/21,
                       0/1/10,
                       1/1/9,
                       2/1/8,
                       3/1/7,
                       4/1/20,
                       0/2/11,
                       1/2/2,
                       3/2/6,
                       4/2/19,
                       0/3/12,
                       1/3/3,
                       2/3/4,
                       3/3/5,
                       4/3/18,
                       0/4/13,
                       1/4/14,
                       2/4/15,
                       3/4/16,
                       4/4/17};
        \def\PotentialColorFactor{4};  
        \foreach \x/\y/\p in \Potential {
            \pgfmathsetmacro\xycolor{\PotentialColorFactor * \p};
            \node[circle, draw=black, fill=transientcolour!\xycolor!convergecolour, inner sep=0pt, minimum size=\CircleSize] (\x\y) at (\Ex * \x, \Ey * \y) {\footnotesize \textbf{\p}};
        }
        \node[circle, ultra thick, draw=black, fill=recurrentcolour, inner sep=0pt, minimum size=\CircleSize] (22) at (\Ex * 2, \Ey * 2) {\footnotesize \textbf{1}};

        \foreach \x/\y/\xx/\yy in \Edges {
            \draw [->, thick, shorten >= 2.5pt, shorten <= 2.5pt] (\x\y) -- (\xx\yy);
        }
    \end{tikzpicture}

    \caption{Spiral chain on a 2-dimensional square lattice---i.e.\ $\SS = \left\{ (x, y) : x \in [4], y \in [4] \right\} \subset \mathbb{Z}^2$---that is finite-state reductive wrt.\ $\SS_\bullet = \left\{ (2, 2) \right\}$ (in {\color{recurrentcolour}blue}). Each node denotes a state, each edge a possible transition, and states are annotated and {\color{transientcolour}co}{\color{transientcolour!90}lo}{\color{transientcolour!80}ur}{\color{transientcolour!70}ed} by their potential $\PF{x} = \left\lvert\SoR{x}\right\rvert = \card{\SoR[_0]{x} \cup \SoR[_1]{x} \cup \cdots}$.}\label{fig:spiral}
\end{figure}

\begin{assumption}[Consistency]\label{ass:consistency}
    The prosets $\{(\SS, \preccurlyeq^\pi_\mu) : \pi\in\Pi\}$ are isomorphic, c.f.~\cite{davey:2002:introduction}, where $\preccurlyeq^\pi_\mu$ denotes the preorder from \eqref{eq:preorder} under a policy $\pi\in\Pi$.
\end{assumption}

\begin{assumption}[Level-Set Oracle]\label{ass:level_set}
    We can compute $\mathcal{M}_\phi\!\left(A\right) \doteq \argmin_{x \in A} \PF{x}$, $A \subseteq \SS$, in $\mathcal{O}(\card{\SS_\circ}\cdot\card{\SS}\cdot\card{\AS})$.
\end{assumption}

\begin{theorem}[RVI]\label{thm:rvi}
    Under Assumptions~\ref{ass:consistency} and~\ref{ass:level_set}, RVI has worst-case complexity $\mathcal{O}(\card{\SS_\circ} \cdot \card{\SS} \cdot \card{\AS}) + \mathcal{O}_\bullet$, where $\mathcal{O}_\bullet$ is the complexity of solving the a.i.\ subspace $\SS_\bullet$. If $r(x, u) = 0$ for all $(x, u)\in\SS_\bullet\times\AS$, then $\mathcal{O}_\bullet = \mathcal{O}(\card{\SS_\bullet})$.
\end{theorem}

Though notably more verbose, one can bound the time complexity of RVI even tighter for the case where $r(x, u) = 0 \; \forall \; (x, u)\in\SS_\bullet\times\AS$. Ignoring the $\card{\SS_\bullet}$ updates required to initialise $v(x)$, solving $\SS_\circ$ takes updates on the order of
\begin{displaymath}
    \sum_{x\in\SS_\circ} \card{\SoR[_1]{x}} \cdot \card{\AS\!\left(x\right)} \leq \card{\SS_\circ} \cdot \card{\AS} \cdot \max_{x\in\SS_\circ} \card{\SoR[_1]{x}}.
\end{displaymath}
This level of complexity is \highlight{very similar to that of breadth-first search} and, indeed, RVI can be seen as an application thereof. Furthermore, since $\max_{x\in\SS_\circ} \card{\SoR[_1]{x}}$ is typically less than $\card{\SS_\circ}$, this relationship confirms that \highlight{RVI will perform particularly well in FRMDPs with low out-degree}; i.e.\ problems where the transition digraph is deep, not wide.

\paragraph{Why follow the potential?}
Consider the FRMC in Figure~\ref{fig:spiral} and augment it with a set of controls $\AS \doteq [0, 1]$, where an action $u \in \AS$ denotes the probability that the agent takes the shorter path when available. For example, in state $(2, 4)$, an action $u = 0$ (resp.\ 1) would move the agent east (south), and any value $u\in(0, 1)$ would mix between these transitions. Applying RVI to this problem reveals some subtleties to the construction. If we were to na\"{i}vely recurse up the chain according to $\mathsf{G}_1$ (see \eqref{eq:g_map}) as in BVI, we would immediately attempt to solve the set $\left\{ (2, 1), (1, 2), (2, 3) \right\}$. However, states $(2, 1)$ and $(2, 3)$ are not yet directly soluble since there are as-yet unsolved paths. In other words, at least one further iteration is required before the true values are reached. The same is true for state $(2, 4)$. This highlights a key advantage of RVI compared with backward VI.

\section{Case Study: Optimal Liquidation}\label{sec:optimal_liq}
Consider a discrete optimal liquidation problem~\cite{almgren:2001:optimal} where the state-space is given by a combination of the agent's inventory $\left\{Q_n : n \geq 0\right\}$ (how much stock is held) and the current market price $\left\{Z_n : n \geq 0\right\}$. For tractability, we constrain ourselves to the finite setting and let $\SS \doteq [0, \overline{q}] \times [\underline{z}, \overline{z}]$ with $0 < \overline{q} < \infty$ and $0 < \underline{z} < \overline{z} < \infty$. The action-space is $\AS \doteq \mathbb{Z}_+$, and we apply state-dependent masking to enforce reductivity: $\AS\!\left(q\right) \doteq \left\{q, q-1, \dots, 1\right\}$ for all $q > 0$, and $\AS\!\left(0\right) \doteq \left\{0\right\}$. The agent's reward is constructed from three mutually dependent objectives:
\begin{displaymath}
    r_{\bm{w}}\!\left(q, z, u\right) = \underbrace{w_0 \cdot u \cdot (z - z_0)}_{\textrm{Excess PnL}} - \underbrace{w_1 \cdot u^2}_{\textrm{Transaction Cost}} - \underbrace{w_2 \cdot q^2}_{\textrm{Inventory Penalty}}
\end{displaymath}
where $\bm{w} \in \Re^3_+$ is to be chosen, and $Z_0 \doteq z_0$ is a fixed initial price. The first term is the excess profit from selling holdings at a time $n \geq 0$; the second is a proxy for the cost of transacting;\footnote{This approximation is used for demonstration. See \cite{abergel:2016:limit} or \cite{cartea:2015:algorithmic} for more realistic alternatives.} and the third is a risk penalty associated with the agent's inventory designed to encourage immediacy.

\paragraph{Transition Dynamics}
At each time, the agent selects an action $0 < U_n \leq Q_n$ and the inventory decreases as $Q_n \mapsto Q_n - U_n$.\footnote{One can incorporate a fill-probability model here, just as long as $Q_{n+1} - Q_n > 0$ holds a.s.\ for all $n \geq 0$ where $Q_n > 0$.} The price dynamics are modelled independently as a trinomial random walk with two-sided reflection~\cite{andersen:2015:levy} and probability vector $\bm{p}_Z\in\Delta^2$; see Figure~\ref{fig:optimal_exec} in the appendix for an illustration. That is, for a driving sequence $\left\{Y_n : n\geq 0\right\}$, $Y_{n+1} \doteq \min\!\left\{ \overline{z} - \underline{z}, \max\!\left\{ 0, Y_n \right\} \right\}$, where $Y_0 \doteq y \in [0, \overline{z} - \underline{z}]$, we define the price at time $n \geq 0$ as $Z_n \doteq \underline{z} + Y_n$.

\paragraph{Level-Set Oracle}
The dynamics outlined above give rise to an FRMDP for which the potential function is non-trivial to express. The supremum of $\PFsymbol$ across prices for a given inventory, $q > 0$, is not bounded by the infimum at $q+1$, prohibiting efficient oracles. A simple (technical) solution is to assume that the random walk has an infinitesimal, but non-zero probability $\varepsilon$ of large price increments (i.e.\ non-trinomial) such that $\SoR{q, z} = \left\{ (q', z') \in \SS : 0 \leq q' < \overline{q}, z \in [\underline{z}, \overline{z}] \right\}$. There are two advantages to this trick: (a) the reachable sets are rectangular, leading to an efficient implementation of RVI that recurses up the tree from $1 \to \overline{q}$, sweeping across the price grid $[a, b]$ at each step (see Figure~\ref{fig:optimal_exec}); and (b) we are free to choose $\varepsilon$ sufficiently small such that any price increments not in $\left\{-1, 0, 1\right\}$ are negligible.

\begin{figure}[t]
    \centering
    \includegraphics[%
        width=0.8\linewidth,%
        trim={0 0.5em 0 0},%
        clip%
    ]{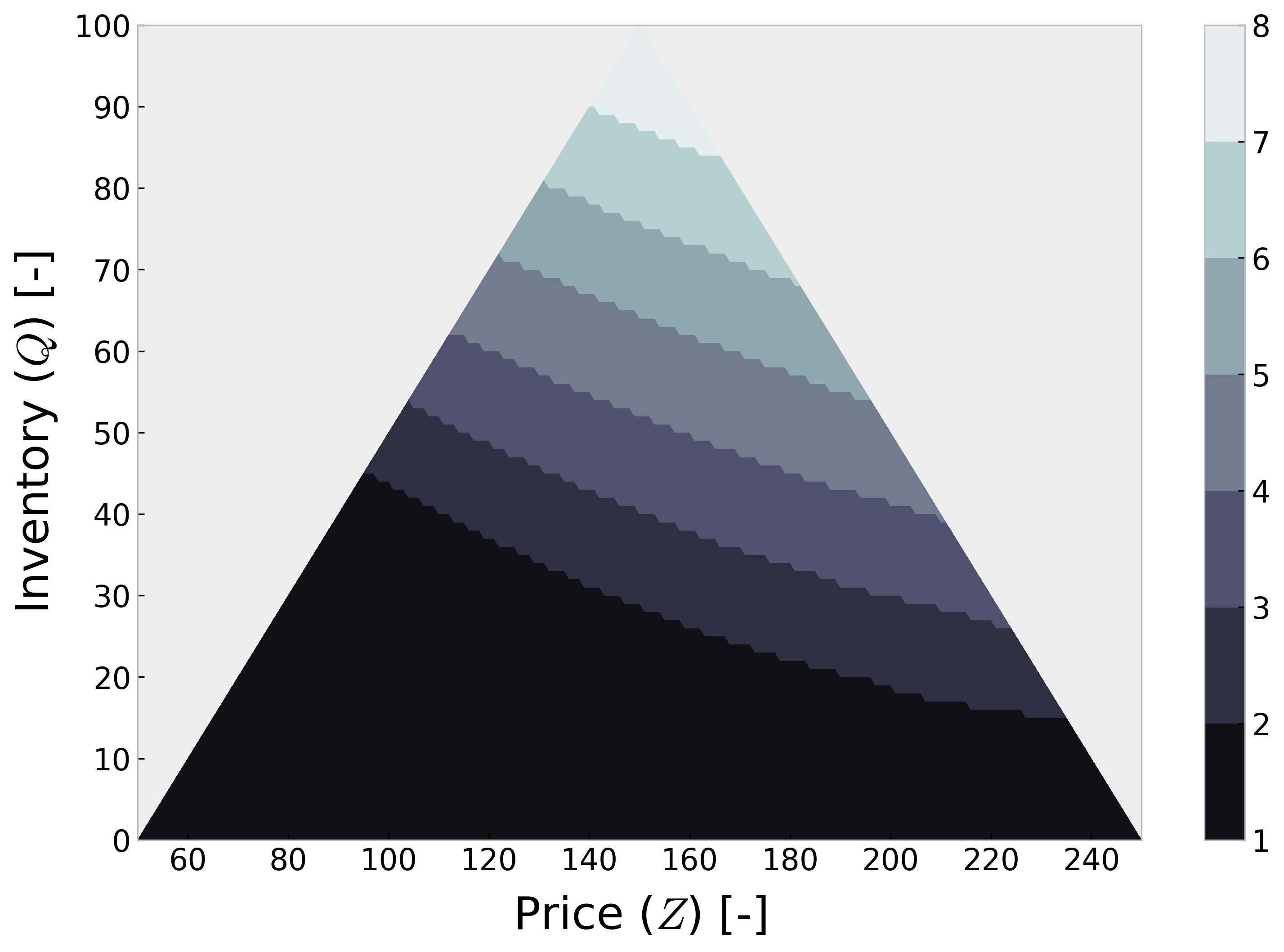}

    \caption{Optimal liquidation policy as a function of state. The $z$-axis denotes the quantity to be sold and masking is applied to the null sets on price evolutions given $Z_0 \doteq 150$.}\label{fig:policy}
\end{figure}

\subsubsection*{Numerical Results}
The first objective was to better understand the nature of optimal liquidation strategies in the discrete setting. To this end, we applied RVI with an initialisation of $Q_0 \doteq \overline{q} \doteq 100$ and $Z_0 \doteq 150 \in \left[50, 250\right]$, price dynamics $\bm{p}_Z \doteq \left[0.4, 0.2, 0.4\right]$, and reward parameters $\bm{w} \doteq \left[1, 0.2, 0.002\right]$. The optimal policy associated with this configuration is illustrated in Figure~\ref{fig:policy}. Note that the smoothness one typically observes in continuous settings~\cite{cartea:2015:algorithmic} is replaced here by a series of level-sets whose boundaries are non-linear in $\SS$, with a bias towards selling more at higher prices. In contrast, the optimal liquidation paths generated by such policies as a function of $w_1$ are extremely well-behaved; see Figure~\ref{fig:paths}. As one might imagine, increasing the transaction penalty leads to a decrease in the aggressiveness of the strategy and a commensurate increase in the length of the schedule.

\begin{figure}[t]
    \centering
    \includegraphics[%
        width=0.95\linewidth,%
        trim={0 0.5em 0 0},%
        clip%
    ]{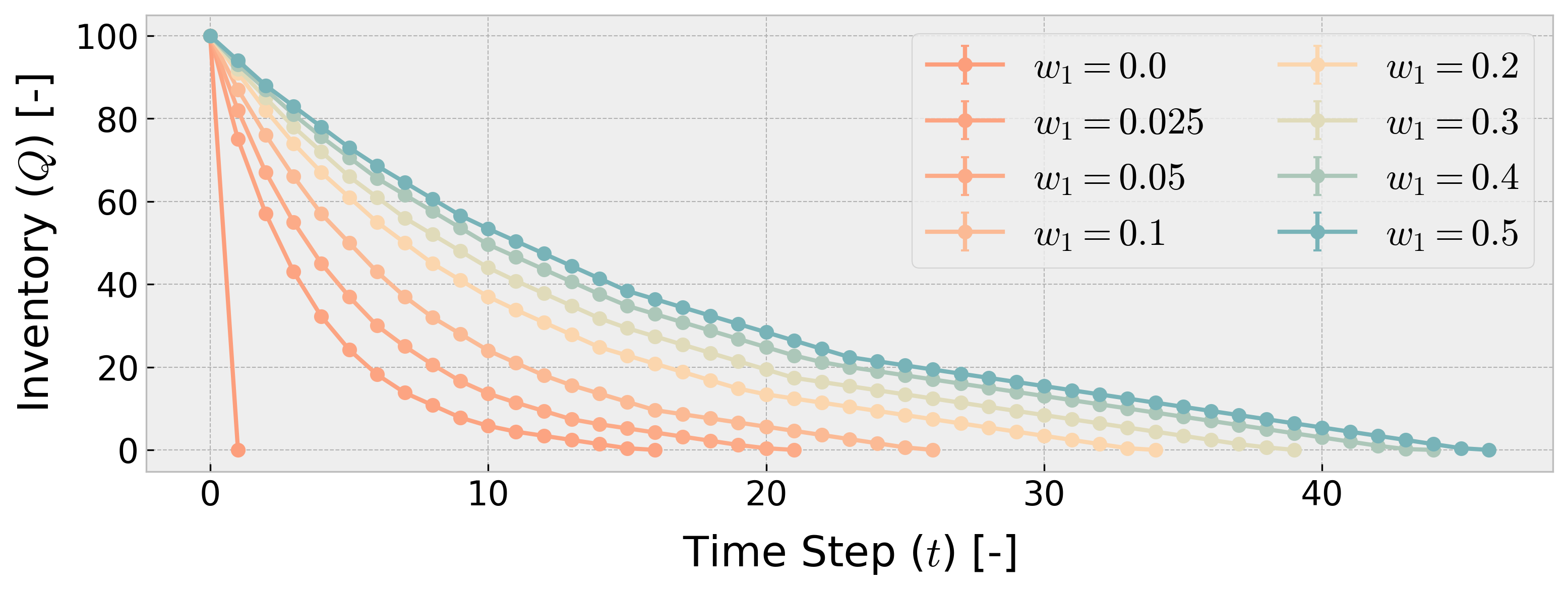}

    \caption{Optimal liquidation trajectories over time as a function of the transaction cost weight, $w_1$. Each point denotes a sample mean, with errors/variance omitted as they are too small to visualise.}\label{fig:paths}
\end{figure}

The second objective was to compare the performance of RVI to existing approaches, of which we consider:
\begin{description}[itemsep=0em]
    \item[QVI] A variant of traditional VI in which the order of processing inventory levels is randomised or the worst-case of recursing from $\overline{q} \to 1$ (i.e.\ reversed RVI).
    \item[BVI] The Backward VI algorithm of \citet{dai:2007:prioritizing}.
\end{description}
The time-complexities shown in Figure~\ref{fig:time_complexity} suggest that \highlight{RVI can achieve as much as a two order of magnitude reduction in elapsed time} to find the optimal policy across values of $\overline{q}$. This results from the fact that RVI is a one-pass algorithm, where other methods may require up to an exponential-in-$\overline{q}$ number of iterations. Interestingly, while BVI essentially performs the same sequence of updates as RVI, the cost of explicitly managing a queue leads to a significant increase in the per-iteration complexity.

\begin{figure}
    \centering
    \includegraphics[%
        width=0.95\linewidth,%
        trim={0 0.5em 0 0},%
        clip%
    ]{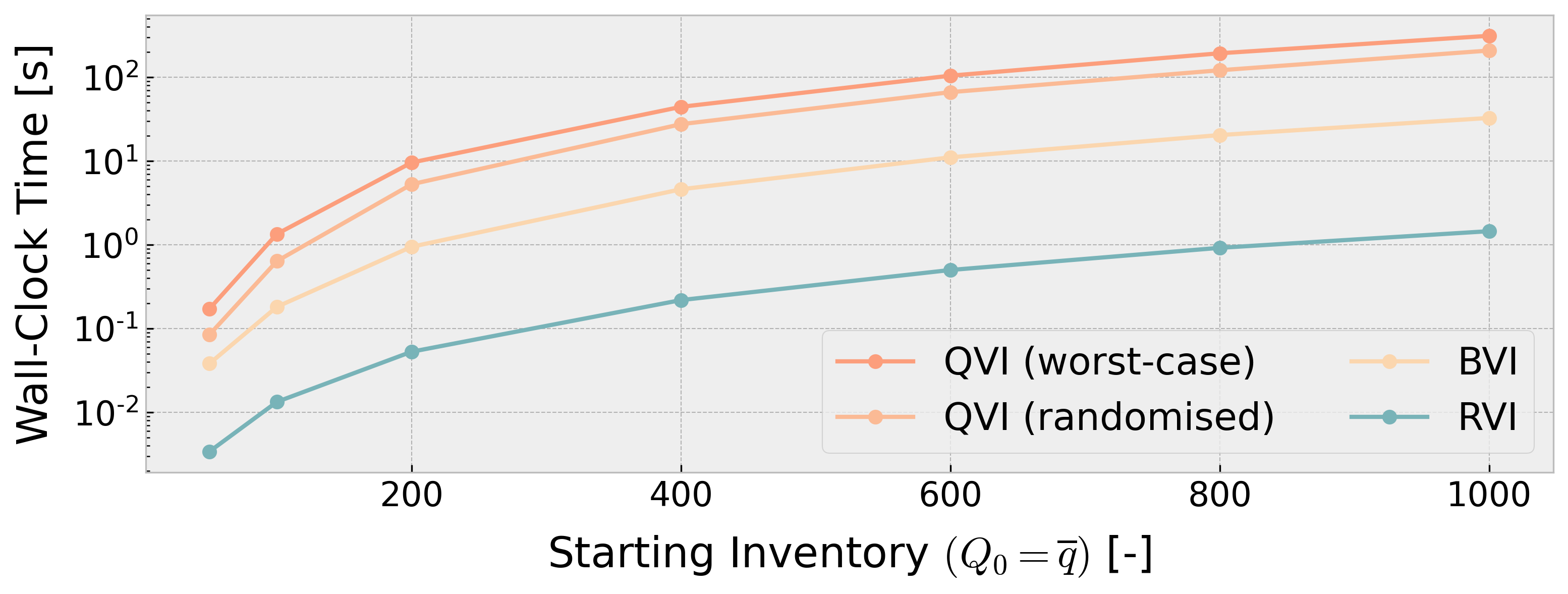}

    \caption{Mean time-complexity of QVI, BVI and RVI on the optimal liquidation domain with $\overline{q} \doteq 100$, $\underline{z} \doteq 40$ and $\overline{z} \doteq 260$.}\label{fig:time_complexity}
\end{figure}

\section{Conclusions and Future Directions}
In this paper we have proposed and analysed a drift condition for Markov chains that induces monotonic and acyclic transition dynamics, generalising finite-horizon problems. It was then shown that this behaviour allows us to extend backwards induction to a much richer class of problems, including finite-resource and optimal routing settings. Below we highlight some particularly interesting future directions.

\paragraph{Verifying Reductivity.}
It is clear from Lemma~\ref{lem:dag_with_loops} that tests for acyclicity of a DAG could be used to identify reductivity of a finite-state MC. Can these algorithms be generalised to arbitrary state-spaces? And indeed, is it possible to design more efficient algorithms for verifying this property in large cardinality spaces, or even in MDPs subject to Assumption~\ref{ass:consistency}? These would likely be hard problems since, for each candidate measure, one would need to evaluate many sub-problems, each of which may be of similar complexity to that of \textsf{EOPL}~\cite{fearnley:2020:unique}.

\paragraph{Weak Reductivity.}
The strict conditions on reductive processes yield significant performance benefits in terms of time-complexity. However, we conjecture that allowing a ``small number of off-upper-triangular terms'' may yield similar benefits by approximating the original process as an RMC. It would be of great interest to establish if this holds, under what conditions, and with what guarantees on bias.

\paragraph{Renewal Reductivity.}
Is it possible to handle renewal points in the context of resource depletion problems? For example, there may be points where a car can refuel in transportation~\cite{chen:2015:stochastic}. For this, one could appeal to renewal theory of Markov chains~\cite{meyn:2012:markov}. This might also offer a route towards identifying the ties between reductive processes and queuing systems.

\section*{Disclaimer}
This paper was prepared for informational purposes by the Artificial Intelligence Research group of JPMorgan Chase \& Co and its affiliates (``J.P.\ Morgan''), and is not a product of the Research Department of J.P.\ Morgan. J.P.\ Morgan makes no representation and warranty whatsoever and disclaims all liability, for the completeness, accuracy or reliability of the information contained herein. This document is not intended as investment research or investment advice, or a recommendation, offer or solicitation for the purchase or sale of any security, financial instrument, financial product or service, or to be used in any way for evaluating the merits of participating in any transaction, and shall not constitute a solicitation under any jurisdiction or to any person, if such solicitation under such jurisdiction or to such person would be unlawful.

\copyright{} 2022 JPMorgan Chase \& Co. All rights reserved.

\bibliography{main}
\bibliographystyle{icml2022}

\clearpage
\appendix
\onecolumn
\section{Some Observations}
\subsection{Level-Set Oracles}
In the black-box case, where a generative model yields $(x, u, x')$ tuples,
finding reachable sets (RSs) is essentially undecidable for general MDPs.
Unless all $(x, x')$ pairs are observed (in which case the RS is $\mathsf{X}$), one can never be sure that the unobserved transitions are just very unlikely. One must instead resort to probabilistic bounds which are beyond the scope of this particular paper, though we conjecture that reductivity would yield tighter results. In the standard setting, where we have access to the transition matrix, $\bm{P}$, there exist efficient algos if $\mathsf{X}_\star$ is given as input to $\mathcal{M}_\phi$. As $\bm{P}$ is upper-triangular in RMDPs, it is sufficient to perform a search over predecessor states from all $x\in\mathsf{X}_\star$ and take those with lowest outdegree that are not yet in $\mathsf{X}_\star$, yielding an algorithm of (very) conservative worst-case complexity of $\mathcal{O}\left(\left\lvert \mathsf{X}_\circ \right\rvert \cdot \left\lvert \mathsf{X} \right\rvert\right)$. This clearly satisfies Assumption 2, and even suggests that tighter bounds are possible. Of course, in many practical applications, such as the liquidation problem, one can implement a much faster oracle.

\section{Theory and Proofs}
\subsection{Supplementary Lemmas}
\begin{lemma}\label{lem:aiae_potential}
    In all a.i.\ subsets $B \in \subsof[_\bullet]{\SS}$ we have that $\SoR[_n]{x} = B \; \forall \, x\in B, n > 0$, and the potential $\PF[_\mu]{x} = \PF[_\mu]{B}$ is uniform across $B$.
\end{lemma}
\begin{proof}
    Consider an a.i.\ set $B$. By the absorbing property, no trajectory of the chain can leave $B$ once it enters. Further, indecomposability means that no path can become permanently restricted to a strict subset $C \subset B$. These properties together imply that the reachable state set is necessarily equal to $B$ uniformly across $B$; that is, $\SoR{x} = B$ for all $x \in B$. The potential must therefore be the same across all such states, concluding the proof.
\end{proof}

\begin{lemma}\label{lem:q_decomp}
    For any MDP with discount factor $\gamma \in [0, 1)$, the action-value function decomposes as \eqref{eq:q_decomp}. That is,
    \begin{align*}
        q(x, u) = r(x, u) + \mhighlight[transientcolour!20]{\gamma\cdot\beta(x, u)\cdot q_\circ'(x, u)} + \mhighlight[recurrentcolour!20]{\gamma\cdot\alpha(x, u)\cdot q_\bullet'(x, u)}
    \end{align*}
    where, for $r(x) \doteq \mathbb{E}_\pi\!\left[R_{t+1} ~\middle\vert~ X_t = x\right]$,
    \begin{align*}
        \alpha(x) &\doteq \cp[_\pi]{X_{t+1} = x}{X_t = x}, \\
        \beta(x) &\doteq \cp[_\pi]{X_{t+1} \ne x}{X_t = x}, \\[0.5em]
        \alpha(x, u) &\doteq \cp[_\pi]{X_{t+1} = x}{X_t = x, U_t = u}, \\
        \beta(x, u) &\doteq \cp[_\pi]{X_{t+1} \ne x}{X_t = x, U_t = u}, \\[0.5em]
        v_\circ'(x) &\doteq \mathbb{E}_\pi\!\left[v(X_{t+1}) ~\middle\vert~ X_t = x, X_{t+1} \ne x\right], \\
        q_\circ'(x, u) &\doteq \mathbb{E}_\pi\!\left[v(X_{t+1}) ~\middle\vert~ X_t = x, U_t = u, X_{t+1} \ne x\right], \\[0.5em]
        q_\bullet'(x, u) &\doteq \mathbb{E}_\pi\!\left[v(x) ~\middle\vert~ X_t = x, U_t = u, X_{t+1} = x\right] = \frac{r(x) + \gamma \cdot \beta(x) \cdot v_\circ'(x)}{1 - \gamma \cdot \alpha(x)},
    \end{align*}
\end{lemma}
\begin{proof}
    The proof follows by considering conditional expectations of mutually disjoint events, as described below.

    \paragraph{Step 1 (Decomposition)}
    We begin by decomposing the state- and action-value functions into two terms: one reflecting the value associated with leaving the state, and the other capturing the value associated with ``self-looping.'' To this end, we first define
    \begin{align*}
        v_\circ(x) &\doteq \mathbb{E}_\pi\!\left[R_{t+1} + \gamma \cdot v(X_{t+1}) ~\middle\vert~ X_t = x, X_{t+1} \ne x\right], \\
        v_\bullet(x) &\doteq \mathbb{E}_\pi\!\left[R_{t+1} + \gamma \cdot v(x) ~\middle\vert~ X_t = x, X_{t+1} = x\right],
    \end{align*}
    and analogously
    \begin{align*}
        q_\circ(x, u) &\doteq \mathbb{E}_\pi\!\left[R_{t+1} + \gamma \cdot v(X_{t+1}) ~\middle\vert~ X_t = x, U_t = u, X_{t+1} \ne x\right], \\
        q_\bullet(x, u) &\doteq \mathbb{E}_\pi\!\left[R_{t+1} + \gamma \cdot v(x) ~\middle\vert~ X_t = x, U_t = u, X_{t+1} = x\right].
    \end{align*}
    From this we can now fully decompose the value functions according to conditional expectation as follows:
    \begin{displaymath}
        v(x) = \underbrace{\mhighlight[transientcolour!20]{\beta(x) \cdot v_\circ(x)}}_{\text{``Transient'' Component}} + \underbrace{\mhighlight[recurrentcolour!20]{\alpha(x) \cdot v_\bullet(x)}}_{\text{``Loop'' Component}},
    \end{displaymath}
    and, for $r(x, u) \doteq \mathbb{E}_\pi\!\left[R_{t+1} ~\middle\vert~ X_t = x, U_t = u\right]$,
    \begin{displaymath}
        q(x, u) = \underbrace{\mhighlight[transientcolour!20]{\beta(x, u) \cdot q_\circ(x, u)}}_{\text{``Transient'' Component}} + \underbrace{\mhighlight[recurrentcolour!20]{\alpha(x, u) \cdot q_\bullet(x, u)}}_{\text{``Loop'' Component}} = r(x, u) + \mhighlight[transientcolour!20]{\beta(x, u) \cdot q_\circ'(x, u)} + \mhighlight[recurrentcolour!20]{\alpha(x, u) \cdot q_\bullet'(x, u)}.
    \end{displaymath}

    \paragraph{Step 2 (Unravel $v(x)$)}
    We want an expression for $q(x, u)$ that can be computed directly. However, $q_\bullet$ and thus $q_\bullet'$ include terms that depend on future iterations occurring at the origin state: namely, $v_\bullet(x)$. To resolve this, we expand $v(x)$ into expressions that can be computed immediately, and those that are recursive:
    \begin{align*}
        v(x)
            &= \mhighlight[transientcolour!20]{\beta(x) \cdot v_\circ(x)} + \mhighlight[recurrentcolour!20]{\alpha(x) \cdot v_\bullet(x)}, \\
            &= \mhighlight[transientcolour!20]{\beta(x) \cdot v_\circ(x)} + \mhighlight[recurrentcolour!20]{\alpha(x) \cdot \mathbb{E}_\pi\!\left[R_{t+1} + \gamma \cdot v(x) ~\middle\vert~ X_t = x, X_{t+1} = x\right]}, \\
            &= \mhighlight[transientcolour!20]{\beta(x) \cdot v_\circ(x)} + \mhighlight[recurrentcolour!20]{\alpha(x) \cdot \left[r_\bullet(x) + \gamma\cdot v(x)\right]},
    \end{align*}
    where $r_\bullet(x) \doteq \mathbb{E}_\pi\!\left[R_{t+1} ~\middle\vert~ X_t = x, X_{t+1} = x\right]$. To simplify further, we recall that $v_\circ(x)$ includes a reward term that is a complement of $r_\bullet(x)$, and thus we can undo the decomposition on the immediate reward to give
    \begin{displaymath}
        v(x) = r(x) + \mhighlight[transientcolour!20]{\gamma\cdot\beta(x)\cdot v_\circ'(x)} + \mhighlight[recurrentcolour!20]{\gamma\cdot\alpha(x)\cdot v(x)}
    \end{displaymath}
    where $v_\circ'(x)$ was defined previously. Now, with a little manipulation, we can show that
    \begin{displaymath}
        \boxed{v(x) = \frac{r(x) + \mhighlight[transientcolour!20]{\gamma\cdot\beta(x)\cdot v_\circ'(x)}}{\mhighlight[recurrentcolour!20]{1 - \gamma\cdot\alpha(x)}}.}
    \end{displaymath}
    Note that this expression for $v(x)$ is no longer recursive, and depends only on those values that can be computed directly from known quantities. Observe also that this quotient is always valid since $\gamma \in [0, 1)$. Indeed, even for $\gamma = 1$, the value $v(x)$ will remain finite as long as the self-loop probability $\alpha(x) < 1$.

    \paragraph{Step 3 (Conclusion)}
    To conclude the proof, we simply recall the decomposition of $q(x, u)$ and replace $q_\bullet'(x, u) = v(x)$ with the (now directly computable) expression derived in the previous step.
\end{proof}

\subsection{Proposition~\ref{prop:completeness}: Completeness}\label{sec:proofs:completeness}
The proof is constructed via a set of contradictions that cover all possible instantiations of the successor set that are not equal to $\SS_\bullet$. The various cases are illustrated in Figure~\ref{fig:completeness} for ease of exposition.

\begin{figure*}
    \centering

    \begin{subfigure}[t]{0.19\textwidth}
        \centering
        \begin{tikzpicture}
            \draw[draw=recurrentcolour, fill=recurrentcolour!20, thick] (1.5cm,0) rectangle (3.0cm,3.0cm);
            \draw[draw=transientcolour, fill=transientcolour!20, thick] (0,0) rectangle (1.5cm,3.0cm);
            \draw (0.35cm,0.25cm) node {$\SS_\bullet$};
            \draw (2.7cm,0.25cm) node {$\SS_\circ$};
            \draw (0.75cm,1.5cm) node {$B$};
        \end{tikzpicture}

        \caption{$B = \SS_\bullet$}\label{fig:completeness:a}
    \end{subfigure}
    \hfill
    \begin{subfigure}[t]{0.19\textwidth}
        \centering
        \begin{tikzpicture}
            \draw[draw=recurrentcolour, fill=recurrentcolour!20, thick] (0,0) rectangle (1.5cm,3.0cm);
            \draw (0.35cm,0.25cm) node {$\SS_\bullet$};
            \draw[draw=recurrentcolour, fill=recurrentcolour!20, thick] (1.5cm,0) rectangle (3.0cm,3.0cm);
            \draw (2.7cm,0.25cm) node {$\SS_\circ$};
            \draw[fill=transientcolour!20, draw=transientcolour, thick] (0,1.5) rectangle (1.5cm,3.0cm);
            \draw (0.75cm,2.25cm) node {$B$};
        \end{tikzpicture}

        \caption{$B \subset \SS_\bullet$}\label{fig:completeness:b}
    \end{subfigure}
    \hfill
    \begin{subfigure}[t]{0.19\textwidth}
        \centering
        \begin{tikzpicture}
            \draw[draw=recurrentcolour, fill=recurrentcolour!20, thick] (0,0) rectangle (1.5cm,3.0cm);
            \draw (0.35cm,0.25cm) node {$\SS_\bullet$};
            \draw[draw=recurrentcolour, fill=recurrentcolour!20, thick] (1.5cm,0) rectangle (3.0cm,3.0cm);
            \draw (2.7cm,0.25cm) node {$\SS_\circ$};
            \draw[fill=transientcolour!20, draw=transientcolour, thick] (0,1.5) rectangle (3.0cm,3.0cm);
            \draw (1.5cm,2.25cm) node {$B$};
        \end{tikzpicture}

        \caption{$B \subset \SS$}\label{fig:completeness:c}
    \end{subfigure}
    \hfill
    \begin{subfigure}[t]{0.19\textwidth}
        \centering
        \begin{tikzpicture}
            \draw[draw=recurrentcolour, fill=recurrentcolour!20, thick] (0,0) rectangle (1.5cm,3.0cm);
            \draw (0.35cm,0.25cm) node {$\SS_\bullet$};
            \draw[draw=recurrentcolour, fill=recurrentcolour!20, thick] (1.5cm,0) rectangle (3.0cm,3.0cm);
            \draw (2.7cm,0.25cm) node {$\SS_\circ$};
            \draw[fill=transientcolour!20, draw=transientcolour, thick] (1.5,1.5) rectangle (3.0cm,3.0cm);
            \draw (2.25cm,2.25cm) node {$B$};
        \end{tikzpicture}

        \caption{$B \subset \SS_\circ$}\label{fig:completeness:d}
    \end{subfigure}
    \hfill
    \begin{subfigure}[t]{0.19\textwidth}
        \centering
        \begin{tikzpicture}
            \draw[draw=recurrentcolour, fill=recurrentcolour!20, thick] (0,0) rectangle (1.5cm,3.0cm);
            \draw (0.35cm,0.25cm) node {$\SS_\bullet$};
            \draw[draw=transientcolour, fill=transientcolour!20, thick] (1.5cm,0) rectangle (3.0cm,3.0cm);
            \draw (2.7cm,0.25cm) node {$\SS_\circ$};
            \draw (2.25cm,1.5cm) node {$B$};
        \end{tikzpicture}

        \caption{$B = \SS_\circ$}\label{fig:completeness:e}
    \end{subfigure}

    \caption{Schematic illustration of the proof construction for Proposition~\ref{prop:completeness}; see Section~\ref{sec:proofs:completeness}.}
    \label{fig:completeness}
\end{figure*}
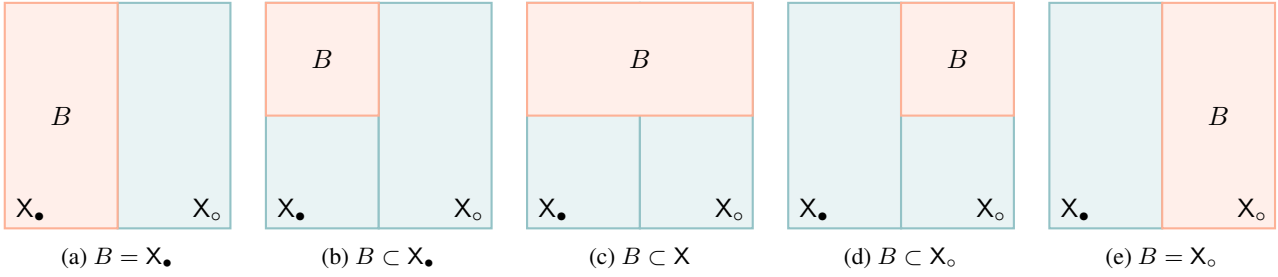

\begin{proof}
    Suppose that the candidate drift condition were defined using a set $B \subset \SS$ containing at least one transient state (i.e.\ $B \cap \SS_\circ \ne \emptyset$); see Figures~\ref{fig:completeness:c}-\ref{fig:completeness:e}. Let $x \in B\cap\SS_\circ$ denote one such state. Then, by Theorem~\ref{thm:decomp} and the definition of inessential sets, the probability that $x$ visits $\SS_\circ$ infinitely often is zero, and thus any chain $\bm{X}$ starting in state $x$ must eventually leave $\SS_\circ$; i.e.\ Corollary~\ref{corr:convergence}. Since $x\in\SS_\circ$, this implies that $\SoR{x}$ and the corresponding potential (by additivity of measures) must reduce in size for at least one state in $B$ which conflicts with \eqref{eq:rmcs}.

    Suppose instead that the candidate set $B \subset \SS_\bullet$ failed to cover the full a.i.\ subspace $\SS_\bullet$; see Figure~\ref{fig:completeness:b}. By Theorem~\ref{thm:decomp}, we know that $\SS_\bullet$ comprises a union of a.i.\ subsets, and thus any state $x \in \SS_\bullet \setminus B$ must occupy one such subset which we denote $C \in \subsof[_\bullet]{\SS}$. It follows from Lemma~\ref{lem:aiae_potential} that $\SoR{x} = C$ and thus there exists at least one successor state $x' \in C$ for which the potential is unchanged, leading to a violation of the drift condition.
\end{proof}

\subsection{Lemma~\ref{lem:dag_with_loops}: DAG with Self-Loops}
\begin{proof}
    We first appeal, as ever, to Theorem~\ref{thm:canonical} to assert that $\SS_\circ$ and $\SS_\bullet$ are disjoint. Then, for an FRMC, the transient part $\SS_\circ$ can be seen to induce a directed graph $\mathcal{G} \doteq \left(\SS_\circ, \mathcal{E}\right)$ with edge set $\mathcal{E} \doteq \left\{ \left(x, x'\right) : x\in\SS_\circ, x'\in\SoR[_1]{x}\setminus\SS_\bullet \right\}$; i.e. the transition graph excluding self-loops. We now show that $\mathcal{G}$ is acyclic.

    Let $D\!\left(x\right)$ denote the descendants of the state $x$ in the graph $\mathcal{G}$. That is, the states $x'\in\SS_\circ$ such that there exists a path from $x$ to $x'$. Note that $D\!\left(x\right) = \SoR{x} \setminus \SoR[_0]{x}$ by definition of an FRMC (see Definition~\ref{def:frmc}). Further, the drift condition \eqref{eq:rmcs} requires that, for every state $x\in\SS_\circ$, each of its descendants $x'\in D\!\left(x\right)$ has the same number or fewer of descendants than $x$ does: $\card{D\!\left(x'\right)} \leq \card{D\!\left(x\right)}$ for all $x\in\SS_\circ$ and $x'\in D\!\left(x\right)$. We also know that for all $x'\in D\!\left(x\right)$, $D\!\left(x'\right)\subseteq D\!\left(x\right)$ by construction of $\mathcal{E}$. However, if the graph $\mathcal{G}$ has a cycle, then there must exist two states $x, x'$ such that $x\in D\!\left(x'\right)$ and $x' \in D\!\left(x\right)$, which implies that $D\!\left(x\right) \subseteq D\!\left(x'\right)$ and $D\!\left(x'\right)\subseteq D\!\left(x\right)$, which further implies that $D\!\left(x\right) = D\!\left(x'\right)$, violating the definition of reductivity and thus concluding the proof.
\end{proof}

\subsection{Corollary~\ref{corr:canon_to_reduct}}
\begin{proof}
    The proof is in fact immediate from the construction used in the proof of Theorem~\ref{thm:canonical}.
\end{proof}

\subsection{Theorem~\ref{thm:rvi}: RVI}
\begin{proof}
    We begin by proving consistency of the algorithm, and then derive a bound on its sample complexity.

    \paragraph{Consistency.}
    Let us first define the sequence $A_n \doteq \argmin_{x\in\SS\setminus A_{n-1}} \PF[_\mu]{x}$ inductively for $n > 0$ with $A_0 \doteq \SS_\bullet$. Given Assumption~\ref{ass:consistency}, these sets are the same for all policies $\pi\in\Pi$. Note also that, for FRMDPs, $\mu$ is the counting measure and thus $\SoR{x}$ is the set of descendants from the state $x\in\SS$. This means that for each $n>0$, $A_n$ corresponds to the least common ancestors of the set $A_{n-1}$~\cite{ait:1989:efficient}. Assuming that $A_0 = \SS_\bullet$ is already computed, then this implies that $v\!\left(x\right)$ is known from the previous iteration for all $x\in\SoR{A_n}$ and thus all updates are consistent on their first iteration.

    \paragraph{Complexity.}
    Since $\cup_{n \geq 0} A_n = \SS$, the value update in \eqref{eq:q_decomp} is performed at most $\card{\SS_\circ}$ times. It thus remains to establish the complexity of evaluating \eqref{eq:q_decomp}. First note that, for a given pair $(x, u)$, computing $r(x, u)$, $\beta(x, u)$, $\alpha(x, u)$ and $q'_\bullet(x, u)$ can all be done in constant time when we have access to the model \emph{and} previously solved states; this holds via the consistency property of RVI. Second, observe that $q_\circ'(x, u)$ can be computed (in the worst-case) via explicit enumerate of all action-next-state tuples that could arise from the incident state $x$. This leads to a complexity of $\mathcal{O}\!\left(\card{\SS}\cdot\card{\AS}\right)$, and thus the same holds for \eqref{eq:q_decomp} altogether. Now, given Assumption~\ref{ass:level_set}, we can also compute the sets $A_n$ with complexity at most $\mathcal{O}\!\left(\card{\SS_\circ}\cdot\card{\SS}\cdot\card{\AS}\right)$. Recalling that $A_n \subseteq \SS_\circ$ for all $n > 0$, we have that the overall complexity must also be $\mathcal{O}\!\left(\card{\SS_\circ}\cdot\card{\SS}\cdot\card{\AS}\right)$ by the summation rule of asymptotic analysis. This concludes the proof.
\end{proof}

\section{Shrinking Intervals}\label{sec:shrinking_intervals}
The \shrinkingintervals{} problem used throughout the paper may be simple to analyse, but it highlights some of the key concepts of reductivity. To better understand this domain, recall that the problem is specified generally as having a state-space $\SS \doteq [0, 1]$ and constrained transition kernel $\SoR[_1]{x} \subseteq [0, x)$ for all $x\in\SS$. For simplicity, let us consider the case where $\SoR[_1]{x} = [0, x - \delta]$ for small $\delta > 0$---implying $\SoR{x} = [0, x - \delta] \cup \{x\}$---and ascribe a uniform probability distribution over said successor states; see Figure~\ref{fig:shrinking_intervals}. In this case, we can explicitly compute the transition function, $\DPF{x}{x'} = x - x'$, and it's conditional expectation:
\begin{align*}
    \mathbb{E}\!\left[\DPF{x}{X'} ~\middle\vert~ X=x\right]
        &= x - \mathbb{E}\!\left[X' ~\middle\vert~ X=x\right], \\
        &= x - \frac{x - \delta}{2}, \\
        &= \frac{x + \delta}{2}.
\end{align*}
This means that, on average, the potential reduces by factor, approximately, of $1/2$ at each step. In fact, if we take the multiplicative formulation described in Section~\ref{sec:rmcs:convergence}, then we can say something stronger as shown in Lemma~\ref{lem:shrinking} below. While this is superseded by Corollary~\ref{corr:convergence}, it is informative to consider the reasons why this is true for the \shrinkingintervals{} case.

\begin{figure}
    \centering
    \begin{tikzpicture}[xscale=6, yscale=4]
        \def\ddelta{0.075};  
        \draw[help lines, color=gray!30, dashed, step=.25] (0,0) grid (1.14,1.14);
        \draw[->, thick] (-.15,0)--(1.15,0) node[right] {$x$};
        \draw[->, thick] (0,-.15)--(0,1.15) node[above] {$\pi\!\left(\cdot ~ \middle\vert ~ x\right)$};

        \foreach \x/\xtext in {0.25/\frac{1}{4}, 0.5/\frac{1}{2}, .75/\frac{3}{4}, 1/1}
        \draw[shift={(\x,0)}] (0pt,1pt) -- (0pt,-1pt) node[below] {$\xtext$};

        \foreach \y/\ytext in {0.25/\frac{1}{4}, 0.5/\frac{1}{2}, .75/\frac{3}{4}, 1/1}
        \draw[shift={(0,\y)}] (0.5pt,0pt) -- (-0.5pt,0pt) node[left] {$\ytext$};

        \node at (-0.05, -0.08) {0};
        \node[circle, draw=black, fill=black, inner sep=0pt, minimum size=3] () at (0, 0) {};

        \foreach \x in {0.25, 0.5, 0.75, 1}
        \draw[line width=2mm, color=black] (\x, 0) -- (\x, \x - \ddelta) node[] {};

        \foreach \x in {0.25, 0.5, 0.75, 1}
        \draw [decorate, decoration={brace, amplitude=2pt, raise=4pt}, yshift=0pt] (\x, \x) -- (\x, \x - \ddelta) node [black, midway, xshift=0.5cm] {\footnotesize $\delta$};

    \end{tikzpicture}

    \caption{Illustration of the \shrinkingintervals{} problem for $\SoR[_1]{x} = [0, x-\delta]$ and $\delta > 0$. The $x$-axis corresponds to the state-value, and the $y$-axis the support of feasible policies in the MDP formulation.}\label{fig:shrinking_intervals}
\end{figure}
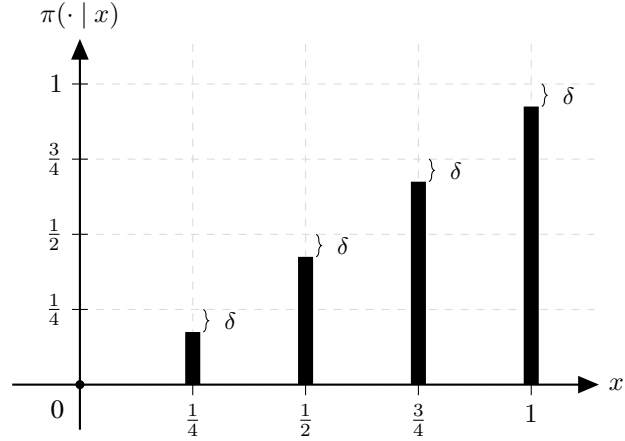

\begin{lemma}\label{lem:shrinking}
    Consider the chain $X_n \doteq \prod_{k=1}^n Y_k$ over independent and identically distributed continuous random variables $Y_n$. If $Y_1$ (and thus all $Y_k$) has support on a subset of the open interval $(-1, 1)$, then $\{X_n\}$ converges a.s.\ to zero.
\end{lemma}
\begin{proof}
    First note that $\abs{X_n} \leq \varepsilon$ if $\abs{Y_k} \leq \varepsilon$ for at least one $1 \leq k \leq n$ with $0 < \varepsilon < 1$. Denoting the latter set of events by $A$, we have that
    \begin{displaymath}
        \p{\abs{X_n} \leq \varepsilon} \geq \p{A} \implies \p{\abs{X_n} > \varepsilon} \leq \p{\lnot A}.
    \end{displaymath}
    Since $\bm{Y}$ is a sequence of independent random variables,
    \begin{displaymath}
        \p{\lnot A} = \prod_{k=1}^n \left[1 - F\!\left(\varepsilon\right)\right] = \left[1 - F\!\left(\varepsilon\right)\right]^n
    \end{displaymath}
    where $F : [0, 1) \to [0, 1]$ denotes the CDF of $\abs{Y_1}$. This implies that the following series is convergent:
    \begin{displaymath}
        \sum_{n=1}^\infty \p{\abs{X_n} > \varepsilon} \leq \sum_{n=1}^\infty \left[1 - F\!\left(\varepsilon\right)\right]^n < \infty,
    \end{displaymath}
    since $\varepsilon > 0$ by construction and thus $F\!\left(\varepsilon\right) > 0$. It follows from the Borel-Cantelli lemma that the events $\{\abs{X_n} > \varepsilon\}_{n > 0}$ occur finitely often which is a sufficient condition for a.s.\ convergence (see Theorem~4.1.3 of \citet{bremaud:2020:probability}), concluding the proof.
\end{proof}

\clearpage
\begin{figure}
    \centering

    \begin{tikzpicture}[scale=.6,every node/.style={minimum size=1cm},on grid]
        \def\CircleSize{27};
        \def\HeightStep{175};
        \def\E{2.2};

        \begin{scope}[
            yshift=0,every node/.append style={
            yslant=0.5,xslant=-1},yslant=0.5,xslant=-1
            ]
            \def\CenterShift{\E * 2};

            \fill[white,fill opacity=0.9] (-\E/2, -2*\E-\E/2) rectangle (4*\E, 1*\E+\E/2);
            \draw[black, dashed] (-\E/2, -2*\E-\E/2) rectangle (3*\E+\E/2, -\E/2);

            \node[circle, draw=black, inner sep=0pt, minimum size=\CircleSize, opacity=.2] (31a) at (0, 1 * \E - \CenterShift) {\footnotesize 1, $\underline{z}$};
            \node[circle, draw=black, inner sep=0pt, minimum size=\CircleSize, opacity=.2] (31a1) at (1 * \E, 1 * \E - \CenterShift) {\footnotesize 1, $\underline{z}$+1};
            \node[] () at (2 * \E, 1 * \E - \CenterShift) {\footnotesize ...};
            \node[circle, draw=black, inner sep=0pt, minimum size=\CircleSize, opacity=.2] (31b) at (3 * \E, 1 * \E - \CenterShift) {\footnotesize 1, $\overline{z}$};

            \node[circle, draw=black, inner sep=0pt, minimum size=\CircleSize, ultra thick] (30a) at (0, 0 * \E - \CenterShift) {\footnotesize 0, $\underline{z}$};
            \node[circle, draw=black, inner sep=0pt, minimum size=\CircleSize, ultra thick, fill=lightgray] (30a1) at (1 * \E, 0 * \E - \CenterShift) {\footnotesize 0, $\underline{z}$+1};
            \node[] () at (2 * \E, 0 * \E - \CenterShift) {\footnotesize ...};
            \node[circle, draw=black, inner sep=0pt, minimum size=\CircleSize, ultra thick] (30b) at (3 * \E, 0 * \E - \CenterShift) {\footnotesize 0, $\overline{z}$};
        \end{scope}

        \begin{scope}[
            yshift=\HeightStep,every node/.append style={
            yslant=0.5,xslant=-1},yslant=0.5,xslant=-1
            ]
            \def\CenterShift{\E * 2};

            \fill[white,fill opacity=0.9] (-\E/2, -2*\E-\E/2) rectangle (4*\E, 2*\E+\E/2);
            \draw[black, dashed] (-\E/2, -2*\E-\E/2) rectangle (3*\E+\E/2, \E/2);

            \node[circle, draw=black, inner sep=0pt, minimum size=\CircleSize, opacity=.2] (32a) at (0, 2 * \E - \CenterShift) {\footnotesize 2, $\underline{z}$};
            \node[circle, draw=black, inner sep=0pt, minimum size=\CircleSize, opacity=.2] (321) at (1 * \E, 2 * \E - \CenterShift) {\footnotesize 2, $\underline{z}$+1};
            \node[] () at (2 * \E, 2 * \E - \CenterShift) {\footnotesize ...};
            \node[circle, draw=black, inner sep=0pt, minimum size=\CircleSize, opacity=.2] (32b) at (3 * \E, 2 * \E - \CenterShift) {\footnotesize 2, $\overline{z}$};

            \node[circle, draw=black, inner sep=0pt, minimum size=\CircleSize,fill=lightgray] (31a) at (0, 1 * \E - \CenterShift) {\footnotesize 1, $\underline{z}$};
            \node[circle, draw=black, inner sep=0pt, minimum size=\CircleSize] (31a1) at (1 * \E, 1 * \E - \CenterShift) {\footnotesize 1, $\underline{z}$+1};
            \node[] () at (2 * \E, 1 * \E - \CenterShift) {\footnotesize ...};
            \node[circle, draw=black, inner sep=0pt, minimum size=\CircleSize] (31b) at (3 * \E, 1 * \E - \CenterShift) {\footnotesize 1, $\overline{z}$};

            \node[circle, draw=black, inner sep=0pt, minimum size=\CircleSize, ultra thick] (30a) at (0, 0 * \E - \CenterShift) {\footnotesize 0, $\underline{z}$};
            \node[circle, draw=black, inner sep=0pt, minimum size=\CircleSize, ultra thick] (30a1) at (1 * \E, 0 * \E - \CenterShift) {\footnotesize 0, $\underline{z}$+1};
            \node[] () at (2 * \E, 0 * \E - \CenterShift) {\footnotesize ...};
            \node[circle, draw=black, inner sep=0pt, minimum size=\CircleSize, ultra thick] (30b) at (3 * \E, 0 * \E - \CenterShift) {\footnotesize 0, $\overline{z}$};

        \end{scope}

        \begin{scope}[
            yshift=2*\HeightStep,every node/.append style={
            yslant=0.5,xslant=-1},yslant=0.5,xslant=-1
            ]
            \def\CenterShift{\E * 2};
            \fill[white,fill opacity=0.9] (-\E/2, -2*\E-\E/2) rectangle (4*\E, 2*\E+\E/2);
            \draw[black,  dashed] (-\E/2, -2*\E-\E/2) rectangle (3*\E+\E/2, \E+\E/2);


            \node[circle, draw=black, inner sep=0pt, minimum size=\CircleSize] (43a) at (0, 3 * \E - \CenterShift) {\footnotesize 3, $\underline{z}$};
            \node[circle, draw=black, inner sep=0pt, minimum size=\CircleSize] (43a1) at (1 * \E, 3 * \E - \CenterShift) {\footnotesize 3, $\underline{z}$+1};
            \node[] () at (2 * \E, 3 * \E - \CenterShift) {\footnotesize ...};
            \node[circle, draw=black, inner sep=0pt, minimum size=\CircleSize] (43b) at (3 * \E, 3 * \E - \CenterShift) {\footnotesize 3, $\overline{z}$};

            \node[circle, draw=black, inner sep=0pt, minimum size=\CircleSize] (42a) at (0, 2 * \E - \CenterShift) {\footnotesize 2, $\underline{z}$};
            \node[circle, draw=black, inner sep=0pt, minimum size=\CircleSize,fill=lightgray] (421) at (1 * \E, 2 * \E - \CenterShift) {\footnotesize 2, $\underline{z}$+1};
            \node[] () at (2 * \E, 2 * \E - \CenterShift) {\footnotesize ...};
            \node[circle, draw=black, inner sep=0pt, minimum size=\CircleSize] (42b) at (3 * \E, 2 * \E - \CenterShift) {\footnotesize 2, $\overline{z}$};

            \node[circle, draw=black, inner sep=0pt, minimum size=\CircleSize] (41a) at (0, 1 * \E - \CenterShift) {\footnotesize 1, $\underline{z}$};
            \node[circle, draw=black, inner sep=0pt, minimum size=\CircleSize] (41a1) at (1 * \E, 1 * \E - \CenterShift) {\footnotesize 1, $\underline{z}$+1};
            \node[] () at (2 * \E, 1 * \E - \CenterShift) {\footnotesize ...};
            \node[circle, draw=black, inner sep=0pt, minimum size=\CircleSize] (41b) at (3 * \E, 1 * \E - \CenterShift) {\footnotesize 1, $\overline{z}$};

            \node[circle, draw=black, inner sep=0pt, minimum size=\CircleSize, ultra thick] (40a) at (0, 0 * \E - \CenterShift) {\footnotesize 0, $\underline{z}$};
            \node[circle, draw=black, inner sep=0pt, minimum size=\CircleSize, ultra thick] (40a1) at (1 * \E, 0 * \E - \CenterShift) {\footnotesize 0, $\underline{z}$+1};
            \node[] () at (2 * \E, 0 * \E - \CenterShift) {\footnotesize ...};
            \node[circle, draw=black, inner sep=0pt, minimum size=\CircleSize, ultra thick] (40b) at (3 * \E, 0 * \E - \CenterShift) {\footnotesize 0, $\overline{z}$};

        \end{scope}

        \begin{scope}[
            yshift=3*\HeightStep,every node/.append style={yslant=0.5, xslant=-1}, yslant=0.5, xslant=-1]]

            \draw[black,dashed] (-\E / 2,2 * \E + -\E / 2) rectangle (3 * \E + \E / 2,2 * \E + \E/2);
            \node[circle, draw=black, inner sep=0pt, minimum size=\CircleSize] (5a) at (0, 2 * \E) {\footnotesize 4, $\underline{z}$};
            \node[circle, draw=black, inner sep=0pt, minimum size=\CircleSize,fill=lightgray] (54a1) at (1 * \E, 2 * \E) {\footnotesize 4, $\underline{z} + 1$};
            \node[] () at (2 * \E, 2 * \E) {\footnotesize ...};
            \node[circle, draw=black, inner sep=0pt, minimum size=\CircleSize] (54b) at (3 * \E, 2 * \E) {\footnotesize 4, $\overline{z}$
            };
        \end{scope}

    \end{tikzpicture}

    \caption{Depicts the decision process associated with an instance of optimal liquidation problem, assuming $\overline{q} = 4$. The dashed squares contain all the states that could be potentially reachable at that decision step (notice how this space necessarily reduces with each decision step). Nodes filled in gray depict one possible execution of the problem with $Q_n = \{4, 2, 1, 0\}$ and $Z_n = \{\underline{z} + 1, \underline{z} + 1, \underline{z}, \underline{z} + 1\}$. Nodes slightly faded out correspond to nodes that cannot be reached under this execution.}
    \label{fig:optimal_exec}
\end{figure}
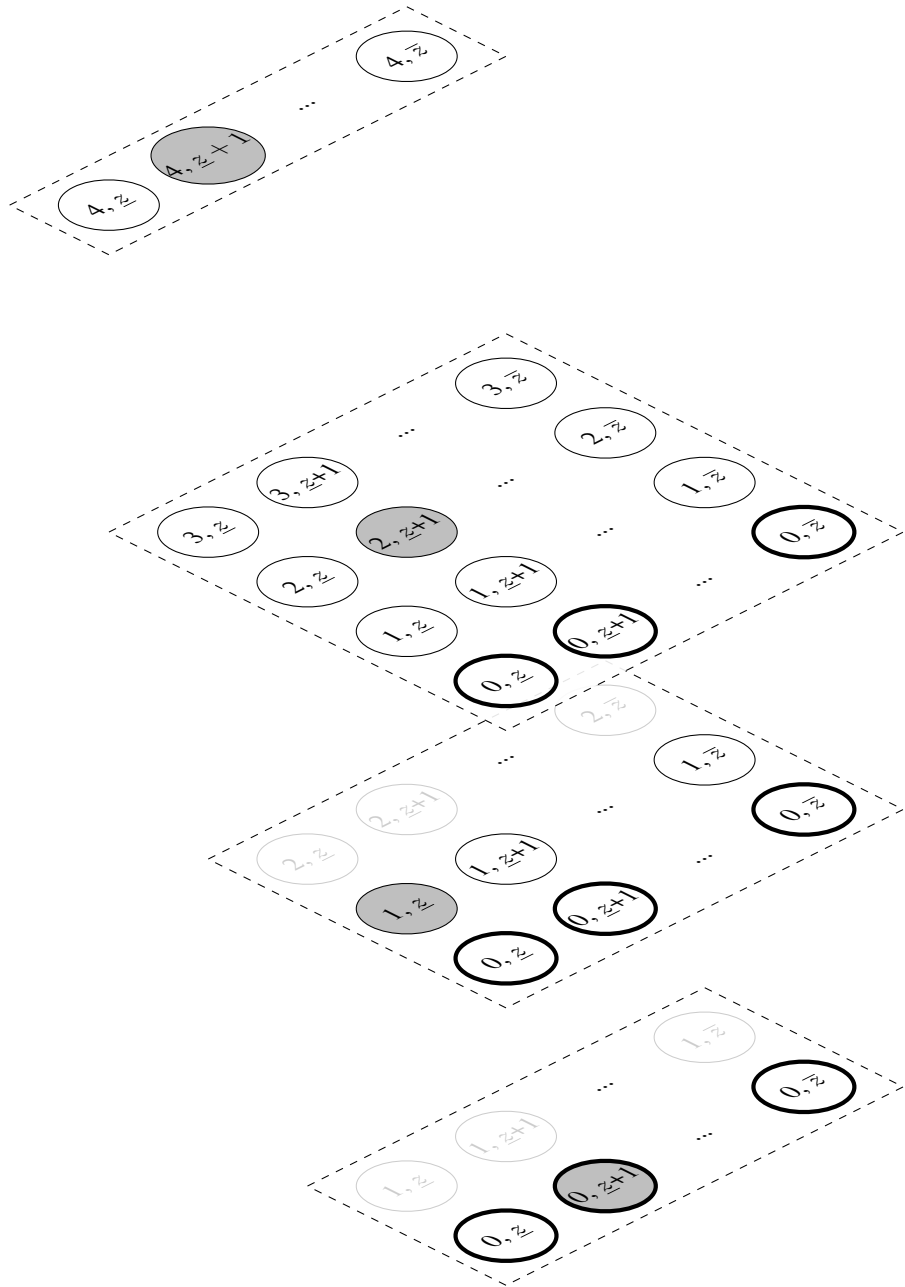

\end{document}